\documentclass[letterpaper,10pt]{article}

\usepackage[preprint]{probml}

\ShortHeadings{Characterizing the Representational Capacity of Neural Processes}

\usepackage{blindtext}
\usepackage{wrapfig}
\usepackage{tikz}
\usetikzlibrary{positioning}
\usepackage{enumitem}

\newtheorem{assumption}{Assumption}
 
\newcommand{\calX}{\mathcal{X}}
\newcommand{\calY}{\mathcal{Y}}
\newcommand{\calC}{\mathcal{C}}
\newcommand{\calF}{\mathcal{F}}
\newcommand{\calN}{\mathcal{N}}
\newcommand{\calP}{\mathcal{P}}

\newcommand{\bK}{\mathbf{K}}
\newcommand{\bI}{\mathbf{I}}
\DeclareMathOperator{\rank}{rank}
\DeclareMathOperator{\vech}{vech}
\DeclareMathOperator{\diag}{diag}

\usepackage{bm,amssymb,amsmath}

\makeatletter
\def\th@plain{%
  \thm@notefont{}
  \itshape 
}
\def\th@definition{%
  \thm@notefont{}
  \normalfont 
}
\makeatother

\def\1{\bm{1}}

\DeclareMathAlphabet{\mathsfit}{\encodingdefault}{\sfdefault}{m}{sl}
\SetMathAlphabet{\mathsfit}{bold}{\encodingdefault}{\sfdefault}{bx}{n}

\newcommand{\E}{\mathbb{E}}

\newcommand{\R}{\mathbb{R}}

\begin{document}

\title{
  Characterizing the Representational Capacity of Neural Processes
}

\author[1,$\dagger$]{Robin Young}

\affil[1]{University of Cambridge}
\affil[$\dagger$]{Correspondence to \url{robin.young@cl.cam.ac.uk}}

\maketitle

\begin{abstract}
What functions can Neural Processes represent? We analyze the representational capacity of popular NP architectures: Conditional Neural Processes (CNPs), Attentive Neural Processes (ANPs), Transformer Neural Processes (TNPs), and their latent variants. We prove these architectures form a strict hierarchy. CNP-representable functions are exactly those depending on finitely many expected features of the context distribution. ANPs strictly generalize CNPs via query-dependent reweighting, enabling kernel smoothers. ConvCNPs and ANPs are incomparable; each contains functions outside the other, separated by stationarity versus translation equivariance. TNPs with $L$ self-attention layers capture $L$-hop context interactions. For latent NPs, we show finite-dimensional latents provide coherent sampling but do not circumvent encoder limitations; matching GP posterior distributions requires latent dimension scaling with context size. These results provide a theoretical foundation for architecture selection based on task structure.
\end{abstract}

\section{Introduction}
 
Neural Processes \citep{garnelo2018cnp, garnelo2018neuralprocesses} have emerged as a family of models for meta-learning and few-shot prediction. By learning to map context sets to predictive distributions, NPs combine the computational efficiency of neural networks with the uncertainty quantification traditionally associated with Gaussian Processes (GPs).
 
Since their introduction, numerous architectural variants have been proposed including Conditional Neural Processes (CNPs), Attentive Neural Processes (ANPs) \citep{kim2018attentive}, Convolutional CNPs \citep{Gordon2020Convolutional}, and Transformer Neural Processes (TNPs) \citep{nguyen2022transformernp}. Many have observed that these architectures exhibit different capabilities. ANPs outperform CNPs on tasks requiring local adaptation, ConvCNPs excel on spatially structured stationary tasks, while TNPs excel at tasks requiring global coherence, but the theoretical foundations explaining these differences have been lacking.
 
In this paper, we answer a natural question: \emph{What functions can neural processes represent?}
 
We provide a characterization of CNP-representable functions as those depending on finitely many expected features of the empirical context distribution. We prove that ANPs strictly generalize CNPs by enabling query-dependent reweighting, with an explicit construction showing that kernel smoothers are ANP-representable but lie outside the CNP function class. We show that ConvCNPs and ANPs are incomparable in the hierarchy with the separation governed by stationarity versus translation equivariance. For TNPs, we establish that $L$ layers of self-attention capture $L$-hop context interactions, and prove matching upper and lower bounds: $\Theta(\sqrt{\kappa} \log(1/\varepsilon))$ layers are necessary and sufficient to $\varepsilon$-approximate GP posteriors, where $\kappa$ is the kernel matrix condition number. We extend our analysis to latent NPs, showing that finite-dimensional latent variables provide coherent function sampling but do not expand representational capacity beyond what the encoder permits.
 
\section{Related Work}
\label{app:related-work}
 
Neural Processes were introduced by \citet{garnelo2018cnp} as a computationally efficient alternative to Gaussian Processes for meta-learning, with the latent variable extension following in \citet{garnelo2018neuralprocesses}. Subsequent work has developed numerous architectural variants. Attentive Neural Processes \citep{kim2018attentive} replaced mean aggregation with cross-attention, allowing the context representation to depend on the target location. Convolutional Conditional Neural Processes \citep{Gordon2020Convolutional} introduced translation equivariance through convolutional aggregation, achieving strong performance on spatially structured tasks. Transformer Neural Processes \citep{nguyen2022transformernp} applied self-attention to the context set before cross-attention, enabling context points to exchange information. Practitioners have observed that these architectures exhibit qualitatively different capabilities, but the theoretical foundations explaining these differences have been absent. We provide the first rigorous characterization of the representational capacity of each architecture class, proving they form a strict hierarchy.
 
The representation of functions over sets has been studied extensively in the deep learning literature. \citet{zaheer2017deepsets} established that permutation invariant functions can be represented by DeepSets architectures with sum or mean aggregation, while follow up work by \citet{wagstaff2019limitations} proved limitations that no continuous function with fixed-dimension aggregation can be injective on sets of unbounded size. Our work extends these results to the predictive setting where outputs depend on both the context set and a target location. The distinction is that Neural Processes must produce query-dependent predictions, not just set-level summaries. We show that query-dependent reweighting strictly expands representational capacity even when the aggregated dimension is held fixed.
 
Gaussian Processes are standard for uncertainty quantification in regression, but exact inference scales cubically in the number of observations. This has motivated extensive work on sparse approximations, including inducing point methods \citep{snelson2005sgp,titsias2009svgp}, random Fourier features \citep{rahimi2007random}, and structured kernel interpolation \citep{wilson2015kernel}. Neural Processes take a different approach as they learn to amortize GP-like inference through an encoder-decoder architecture trained across tasks. Our results characterize when this amortization is possible. We prove that CNPs and ANPs cannot represent GP posteriors regardless of encoder capacity (Theorems~\ref{thm:gp-not-cnp} and \ref{thm:gp-not-anp}), while TNPs can approximate GP posteriors with depth scaling as $\Theta(\sqrt{\kappa} \log(1/\varepsilon))$ where $\kappa$ is the kernel matrix condition number (Theorem~\ref{thm:main-lower-bound}, Proposition~\ref{prop:chebyshev-upper}). This provides a precise sense in which TNPs succeed where simpler architectures fail.
 
The representational capacity of transformer architectures has received significant recent attention. \citet{Yun2020Are} proved that transformers are universal approximators of sequence-to-sequence functions, while \citet{perez2021attention} established Turing completeness under suitable assumptions. Our depth lower bound (Theorem~\ref{thm:main-lower-bound}) contributes to this literature by proving task-specific depth requirements. Approximating GP posteriors requires $\Omega(\sqrt{\kappa} \log(1/\varepsilon))$ layers regardless of width. The key insight is that linearization at $y_C = 0$ reduces the Jacobian to a polynomial in the attention matrix, enabling the application of classical approximation barriers. This technique may be applicable to other regression tasks requiring matrix inversion.
 
Our depth lower bound relies on classical results from approximation theory, particularly polynomial approximation of rational functions. The Chebyshev barrier for approximating $1/\mu$ on an interval dates to the foundational work of Chebyshev and Markov; see \citet{trefethen2019approximation} for a modern treatment. The connection between iterative methods and polynomial approximation is well-established in numerical linear algebra, where Chebyshev iteration achieves optimal convergence rates for linear systems \citep{golub2013matrix}. Our contribution is recognizing that transformer self-attention layers implement polynomial iterations, and that the same barriers governing classical iterative methods govern neural network depth requirements. The matching upper bound (Proposition~\ref{prop:chebyshev-upper}) shows that the Chebyshev polynomial construction can be realized by appropriate choice of attention weights.

Theoretical analysis of meta-learning has developed along several axes. Generalization bounds for learning-to-learn were established by \citet{baxter2000inductive}, with subsequent refinements using PAC-Bayes \citep{pentina2014paclifelong, rothfuss2021pacoh} and information-theoretic \citep{jose2021infotheorymeta} techniques. These results bound how many tasks are needed for a meta-learner to generalize to new tasks. Our work is complementary. We characterize what can be represented, not how fast it can be learned. The expressiveness hierarchy we establish has implications for meta-learning theory. A CNP cannot benefit from additional tasks if the target predictor lies outside its function class, but a full sample complexity analysis for Neural Processes remains open.

\section{Preliminaries}
\label{sec:prelim}
 
Let $\calX \subseteq \R^{d_x}$ be the input space and $\calY \subseteq \R^{d_y}$ the output space. A context set is a finite collection $C = \{(x_i, y_i)\}_{i=1}^n$ with $x_i \in \calX$, $y_i \in \calY$. The context space $\calC = \bigcup_{n=1}^{\infty} (\calX \times \calY)^n / S_n$ is the set of all finite multisets, with $\calC_{\leq N}$ the restriction to size at most $N$. A \emph{predictive map} $F: \calC \times \calX \to \Theta$ returns distribution parameters $\theta \in \Theta$ given a context set $C$ and target location $x_t$.
 
\paragraph{Conditional Neural Process (CNP).}
A CNP computes:
\begin{align}
h_i &= h_\phi(x_i, y_i) \in \R^d & &\text{(encode each context point)} \\
r_C &= \frac{1}{n} \sum_{i=1}^n h_i \in \R^d & &\text{(aggregate via mean)} \\
\theta &= g_\psi(x_t, r_C) \in \Theta & &\text{(decode at target)}
\end{align}
where $h_\phi: \calX \times \calY \to \R^d$ is the encoder and $g_\psi: \calX \times \R^d \to \Theta$ is the decoder.

\paragraph{Attentive Neural Process (ANP).}
An ANP uses cross-attention from target to context:
\begin{align}
\alpha_i(x_t; C) &= \frac{\exp(q(x_t)^\top k(x_i, y_i) / \tau)}{\sum_{j=1}^n \exp(q(x_t)^\top k(x_j, y_j) / \tau)} \\
r_C(x_t) &= \sum_{i=1}^n \alpha_i(x_t; C) \, v(x_i, y_i) \\
\theta &= g_\psi(x_t, r_C(x_t))
\end{align}
where $q: \calX \to \R^{d_k}$ is the query network, $k: \calX \times \calY \to \R^{d_k}$ is the key network, and $v: \calX \times \calY \to \R^d$ is the value network.

\paragraph{Transformer Neural Process (TNP).}
A TNP applies self-attention to the context before cross-attention:
\begin{align}
\beta_{ij}^{(\ell)} &= \frac{\exp(q_s(\tilde{h}_i^{(\ell-1)})^\top k_s(\tilde{h}_j^{(\ell-1)}) / \tau)}{\sum_{m=1}^n \exp(q_s(\tilde{h}_i^{(\ell-1)})^\top k_s(\tilde{h}_m^{(\ell-1)}) / \tau)} \\
\tilde{h}_i^{(\ell)} &= \sum_{j=1}^n \beta_{ij}^{(\ell)} \, W_v \tilde{h}_j^{(\ell-1)}
\end{align}
for layers $\ell = 1, \ldots, L$, with $\tilde{h}_i^{(0)} = h(x_i, y_i)$. The final representation is then used in cross-attention as in ANP.

\paragraph{Convolutional Neural Process (ConvCNP).}
A ConvCNP replaces finite-dimensional aggregation with functional convolutional channels. A non-negative filter $w: \R^{d_x} \to \R_+$ and pointwise encoder $h: \calY \to \R^d$ produce:
\begin{align}
\rho_C(x) &= \sum_{i=1}^n w(x - x_i) & &\text{(density channel)} \\
s_C(x) &= \sum_{i=1}^n w(x - x_i)\, h(y_i) & &\text{(signal channel)} \\
\tilde{r}_C &= \Phi(s_C, \rho_C) & &\text{(CNN processing)} \\
\theta &= g(\tilde{r}_C(x_t), x_t) & &\text{(readout at target)}
\end{align}
where $\Phi$ is a multi-layer CNN. A pure ConvCNP omits the CNN: $\theta = g(s_C(x_t), \rho_C(x_t), x_t)$. We distinguish between the pure convolutional aggregation and the subsequent CNN processing, as these contribute qualitatively different representational capabilities.

The core question of representational capacity is basically: what information about the context can reach the prediction? For CNPs, only the mean-aggregated encoding and context points are processed independently. For ANPs, the query can reweight context points, but the weights factorize but context points still don't interact. ConvCNPs lift the finite-dimensional bottleneck via functional representations, and the CNN enables context--context coupling, but only through translation-equivariant operations. For TNPs, self-attention lets context points exchange information before the prediction is made. This progression from no interaction, to query-mediated reweighting, to full context-context coupling is what drives the hierarchy we establish.

\section{Conditional Neural Processes}
\label{sec:cnp}
 
\begin{definition}[$h$-Equivalence]
For a fixed encoder $h: \calX \times \calY \to \R^d$, two context sets $C, C'$ are $h$-equivalent, written $C \sim_h C'$, if $\frac{1}{|C|} \sum_{(x,y) \in C} h(x,y) = \frac{1}{|C'|} \sum_{(x,y) \in C'} h(x,y)$.
\end{definition}
 
\begin{proposition}[Indistinguishability]
A CNP with encoder $h$ cannot distinguish $h$-equivalent context sets: if $C \sim_h C'$, then $\mathrm{CNP}(C, x_t) = \mathrm{CNP}(C', x_t)$ for all $x_t \in \calX$.
\end{proposition}
 
\begin{theorem}[CNP Characterization]
\label{thm:cnp-char}
The class of $d$-representable predictive maps is exactly the moment statistics class $\calF_{\mathrm{moment}}^{(d)} = \{ F : F(C, x_t) = \phi( \frac{1}{n}\sum_i \psi(x_i, y_i), x_t ) \}$ for continuous $\psi: \calX \times \calY \to \R^d$, $\phi: \R^d \times \calX \to \Theta$.
\end{theorem}
 
\begin{proposition}[Existence of Collisions]
\label{prop:collisions}
For any encoder $h: \calX \times \calY \to \R^d$ and any $n > d/(d_x + d_y)$, there exist distinct context sets $C \neq C'$ with $|C| = |C'| = n$ such that $C \sim_h C'$.
\end{proposition}

\begin{example}[$n=2$ Collision]
\label{ex:two-point}
Consider $d_x = d_y = 1$ and encoder $h(x,y) = (x, y) \in \R^2$. The contexts $C = \{(0, 1), (2, 1)\}$ and $C' = \{(0.5, 0.5), (1.5, 1.5)\}$ satisfy $\bar{h}_C = \bar{h}_{C'} = (1, 1)$, yet for an RBF kernel, $k(0, 2) \neq k(0.5, 1.5)$, so the GP posterior means differ. This two-point collision reappears in the ANP analysis (Theorem~\ref{thm:gp-not-anp}), where we show that even query-dependent reweighting cannot resolve it.
\end{example}

\begin{theorem}[GP Posterior is Not Finitely Representable]
\label{thm:gp-not-cnp}
A predictive map $F$ is $(d, \varepsilon)$-representable on $\calC_{\leq n}$ if there exist continuous $h: \calX \times \calY \to \R^d$ and $g: \R^d \times \calX \to \Theta$ such that $\sup_{C \in \calC_{\leq n}, x_t \in \calX} \|F(C, x_t) - g(\frac{1}{|C|}\sum_{(x,y) \in C} h(x,y), x_t)\| \leq \varepsilon$. Exact representability corresponds to $\varepsilon = 0$.

For any $d$ and any $n > d$, the GP posterior mean $\mu(x_t | C) = k(x_t, X_C) \bK(X_C, X_C)^{-1} y_C$ is not $(d, 0)$-representable on $\calC_{\leq n}$ for generic positive definite kernels $k$.
\end{theorem}
 
\begin{theorem}[CNP Approximation Lower Bound]
\label{thm:cnp-lower-bound}
Let $\nu$ denote a distribution over target locations $x_t \in \calX$, context locations $x_1, \ldots, x_n \in \calX$ be fixed and let $y_C \sim \calN(0, \bK)$. For the GP posterior mean, any CNP with representation dimension $d < n$ satisfies:
\[
\inf_{\mathrm{CNP}} \frac{\E_{y_C} \int |\mu(x_t | C) - \widehat{\mu}(x_t | C)|^2 \, d\nu(x_t)}{\E_{y_C} \int |\mu(x_t | C)|^2 \, d\nu(x_t)} \geq 1 - \frac{d}{n}
\]
for any target distribution $\nu$ such that the whitened kernel vectors $\{\bK^{-1/2} k(x_t, X_C)\}_{x_t \sim \nu}$ have isotropic second moment.
\end{theorem}

\begin{example}[Linear Regression]
Even ordinary least squares illustrates the bottleneck. With $k$-dimensional features $\psi: \calX \to \R^k$, the OLS predictor $F(C, x_t) = \langle (\sum_i \psi(x_i)\psi(x_i)^\top)^{-1} \sum_i y_i \psi(x_i),\, \psi(x_t) \rangle$ requires encoding both the Gram matrix $\sum_i \psi(x_i)\psi(x_i)^\top$ ($k(k+1)/2$ parameters) and the moment vector $\sum_i y_i \psi(x_i)$ ($k$ parameters), giving $d = k(k+3)/2$. Since distinct Gram matrices generically yield distinct predictors, $d = \Omega(k^2)$ is also necessary.
\end{example}

The bound in Theorem~\ref{thm:cnp-lower-bound} is sharp and achieved by the PCA encoder $\bar{h}_C = A^* y_C$ that projects onto the top $d$ principal components of the whitened weight covariance. For $d \ll n$, most GP posterior structure is lost regardless of encoder sophistication. The isotropic second moment condition holds whenever context and target locations are drawn i.i.d.\ from the same distribution and the kernel is stationary (e.g.\ RBF); for non-isotropic targets the bound generalizes to $\sum_{i=d+1}^n \lambda_i / \sum_{i=1}^n \lambda_i$ where $\lambda_i$ are the eigenvalues of $\E_{x_t \sim \nu}[\tilde{w}_{x_t}\tilde{w}_{x_t}^\top]$. Full proofs appear in Appendix~\ref{app:cnp-proofs}.
 
\section{Attentive Neural Processes}
\label{sec:anp}
 
The key difference in ANPs is that the representation $r_C(x_t)$ depends on the query location $x_t$ via attention weights $\alpha_i(x_t; C) \propto \exp(q(x_t)^\top k(x_i, y_i) / \tau)$. Unlike CNP-equivalence, ANP-equivalence is query-dependent: context sets equivalent at one query may be distinguishable at another.
 
\begin{theorem}[ANPs Represent Kernel Smoothers]
\label{thm:anp-kernel}
For any continuous kernel $K: \calX \times \calX \to \R_+$, the kernel smoother $F(C, x_t) = \sum_i K(x_t, x_i) y_i / \sum_i K(x_t, x_i)$ is ANP-representable to arbitrary precision with value dimension $d = d_y + 1$. 
\end{theorem}
 
\begin{proof}[Proof sketch]
Set $v(x, y) = (y, 1)$, $k(x, y) = \psi(x)$, $q(x_t) = \phi(x_t)$, and choose $\phi, \psi$ so that $q(x_t)^\top k(x_i, y_i) \approx \log K(x_t, x_i)$. Then the attention weights recover the normalized kernel weights, and a linear decoder extracts the kernel smoother. The full construction and universal approximation argument appear in Appendix~\ref{app:anp-proofs}.
\end{proof}
 
\begin{theorem}[ANP Characterization]
\label{thm:anp-char}
A predictive map $F$ is ANP-representable if and only if $F(C, x_t) = G( \E_{P_C^{x_t}}[v(x, y)], x_t )$ where $P_C^{x_t}(x_i, y_i) \propto \exp(s(x_t, x_i, y_i))$ for some continuous score function $s$. 
\end{theorem}
 
\begin{corollary}[Strict Separation]
$\calF_{\mathrm{CNP}}^{(d)} \subsetneq \calF_{\mathrm{ANP}}^{(d)}$ for all $d$.
\end{corollary}
 
\begin{theorem}[GP Posterior Requires Context-Context Coupling]
\label{thm:gp-not-anp}
For generic positive definite kernels, the GP posterior mean is not ANP-representable.
\end{theorem}
 
\begin{proof}[Proof sketch]
The GP posterior weight $w_i = [k(x_t, X_C) \bK^{-1}]_i$ depends on the full Gram matrix, coupling all context points. ANP attention weights factor as $\alpha_i \propto f(x_t, x_i, y_i)$, depending only on query-point pairs independently. The simplest counterexample has $n=2$: the GP weight on $y_1$ is
\[
w_1 = \frac{k(x_t, x_1) k(x_2, x_2) - k(x_t, x_2) k(x_1, x_2)}{k(x_1, x_1) k(x_2, x_2) - k(x_1, x_2)^2},
\]
which depends on the inter-context kernel value $k(x_1, x_2)$. No ANP attention weight $\alpha_1 \propto \exp(q(x_t)^\top k(x_1, y_1))$ can capture this coupling. See Appendix~\ref{app:anp-proofs} for the general argument.
\end{proof}
  
\section{Transformer Neural Processes}
\label{sec:tnp}
 
\subsection{Polynomial Computation via Self-Attention}
 
We analyze TNPs under two structural assumptions that enable clean characterization: position-based self-attention (attention weights depend only on input positions, not representations) and residual connections. These are relaxed in Section~\ref{sec:lower-bounds}. Full definitions appear in Appendix~\ref{app:tnp-assumptions}.
 
\begin{theorem}[Polynomial Representation]
\label{thm:polynomial}
Under position-based attention with residual connections, after $L$ self-attention layers with scalar value weights $W_v^{(\ell)} = \alpha_\ell \bI$, the representation matrix satisfies $H^{(L)} = p_L(\tilde{\bK}) H^{(0)}$ where $p_L(\tilde{\bK}) = \prod_{\ell=1}^{L} (\bI + \alpha_\ell \tilde{\bK})$ is a degree-$L$ polynomial in the attention matrix $\tilde{\bK}$.
\end{theorem}
 
\begin{proof}
By induction: $H^{(L)} = (\bI + \alpha_L \tilde{\bK}) H^{(L-1)} = \prod_{\ell=1}^{L} (\bI + \alpha_\ell \tilde{\bK}) H^{(0)}$.
\end{proof}
 
\subsection{Approximating GP Posteriors}
 
The GP posterior involves $\bK^{-1}$. Via the Neumann series $\bK^{-1} = \frac{1}{\lambda_{\max}} \sum_{m=0}^{\infty} (\bI - \bK/\lambda_{\max})^m$, this can be approximated by matrix polynomials. The truncation error is $O(\rho^L / \lambda_{\min})$ where $\rho = 1 - 1/\kappa$, requiring $L = O(\kappa \log(1/\varepsilon))$ layers. The Chebyshev construction improves this to $O(\sqrt{\kappa} \log(1/\varepsilon))$:
 
\begin{proposition}[Chebyshev Upper Bound]
\label{prop:chebyshev-upper}
There exist scalar value weights $\alpha_1, \ldots, \alpha_L$ such that $\| \prod_{\ell=1}^{L}(\bI + \alpha_\ell \bK) - \bK^{-1} \| \leq \frac{2}{\lambda_{\min}} \left( \frac{\sqrt{\kappa} - 1}{\sqrt{\kappa} + 1} \right)^L$.
\end{proposition}
 
The weights are the optimal Chebyshev iteration parameters $\alpha_\ell = -2/(\lambda_{\max} + \lambda_{\min} + (\lambda_{\max} - \lambda_{\min})\cos\theta_\ell)$ with $\theta_\ell = (2\ell-1)\pi/(2L)$. See Appendix~\ref{app:tnp-upper} for the full proof and idealized TNP approximation construction.
 
\subsection{Lower Bounds for Representation-Based Attention}
\label{sec:lower-bounds}
 
The upper bound assumed position-based attention. We now show the $\Omega(\sqrt{\kappa}\log(1/\varepsilon))$ depth requirement holds even for fully adaptive representation-based attention.
 
The key insight is that the GP posterior mean $\mu(x_t|C) = k(x_t, X_C)\bK^{-1}y_C$ is linear in $y_C$. Linearizing the TNP at $y_C = 0$ neutralizes adaptive routing: the Jacobian $\partial F / \partial y_C|_{y_C=0}$ becomes a polynomial in the base attention matrix (evaluated at $y_C = 0$), regardless of how sophisticated the attention mechanism.
 
\begin{theorem}[TNP Depth Lower Bound]
\label{thm:main-lower-bound}
For any $\kappa > 1$, there exists a family of context configurations with kernel matrices satisfying $\kappa(\bK) = \kappa$ such that any $L$-layer TNP with representation-based self-attention achieving $\varepsilon$-approximation of $\bK^{-1} y_C$ must satisfy:
\[
L \geq \frac{\sqrt{\kappa}}{4} \cdot \log\left( \frac{c}{\varepsilon} \right)
\]
for a universal constant $c > 0$.
\end{theorem}
 
\begin{proof}[Proof sketch]
The argument proceeds in four stages:
 
\emph{(1) Linearization.} By Taylor expansion (Lemma~\ref{lem:linearization} in Appendix~\ref{app:tnp-lower}), $\varepsilon$-approximation of the linear target implies the Jacobian $M(X) = \partial F / \partial y_C|_{y_C=0}$ satisfies $\|M(X) - \bK^{-1}\| \leq O(\varepsilon)$.
 
\emph{(2) Polynomial structure.} Differentiating through $L$ self-attention layers shows $M(X)$ is a matrix polynomial of degree at most $2L$ in the attention matrix $\tilde{\bK}|_{y_C=0}$ (Appendix~\ref{app:tnp-lower}, Corollary~\ref{cor:polynomial-structure}).
 
\emph{(3) Univariate reduction.} We construct a family of kernel matrices $\{\bK_t\}$ (Lemma~\ref{app:lem:eigenvalue-family}) where all eigenvalues except the minimum are equal to 1, with condition number $\kappa$. On this family, the quadratic form $v_1^\top M(\bK_t) v_1$ reduces to a univariate polynomial in the minimum eigenvalue $\mu_1(t)$.
 
\emph{(4) Chebyshev barrier.} The target $v_1^\top \bK_t^{-1} v_1 = 1/\mu_1(t)$ must be approximated by a degree-$2L$ polynomial on $[1/\kappa, 1]$. The classical Chebyshev lower bound gives error at least $\Omega(\rho^{2L})$ where $\rho = (\sqrt{\kappa}-1)/(\sqrt{\kappa}+1)$. Solving for $L$ yields the stated bound.
\end{proof}
 
The full proof, including the Jacobian evolution through representation-based attention layers, the eigenvalue-controlled kernel family construction, and spectral analysis, appears in Appendix~\ref{app:tnp-lower}.
 
\begin{corollary}[Tight Characterization]
\label{cor:tight}
For well-conditioned context geometries, $L = \Theta(\sqrt{\kappa} \log(1/\varepsilon))$ self-attention layers are both necessary and sufficient for $\varepsilon$-accurate GP posterior approximation.
\end{corollary}
 
\begin{remark}[Multi-Head Attention]
The analysis extends to multi-head attention. Linearization at $y_C = 0$ yields a Jacobian that is a sum of products of per-head attention matrices, and approximating $\bK^{-1}$ in any such form faces the same Chebyshev barrier. The $\Theta(\sqrt{\kappa})$ depth scaling is intrinsic to the condition number.
\end{remark}
 
\begin{theorem}[Dimension Scaling]
\label{thm:dimension}
To represent GP posterior predictions on contexts of size $n$: (a) CNP and ANP require infinite $d$; (b) TNP requires $d = O(n)$.
\end{theorem}
 
\section{Convolutional Neural Processes}
\label{sec:convcnp}
 
ConvCNPs replace finite-dimensional aggregation with functional convolutional representations. Two structural properties distinguish them.
 
\begin{proposition}[Translation Equivariance]
\label{prop:convcnp-equivariance}
Every ConvCNP satisfies $F(\{(x_i + \tau, y_i)\}, x_t + \tau) = F(\{(x_i, y_i)\}, x_t)$ for all $\tau \in \R^{d_x}$.
\end{proposition}
 
\begin{proposition}[Injectivity]
\label{prop:convcnp-injective}
If $w$ has non-vanishing Fourier transform and $h$ is injective, then the convolutional aggregation $C \mapsto (s_C(\cdot), \rho_C(\cdot))$ is injective up to permutation on contexts with distinct locations.
\end{proposition}
 
Thus ConvCNPs avoid the collision problem of \Cref{prop:collisions}, but at the cost that any non-translation-equivariant predictive map lies outside the ConvCNP function class. Proofs appear in Appendix~\ref{app:convcnp-proofs}.
 
\begin{proposition}[Pure ConvCNP Represents Stationary Kernel Smoothers]
\label{prop:convcnp-kernel}
For any continuous stationary kernel $K$, the Nadaraya--Watson estimator $F(C, x_t) = \sum_i K(x_t - x_i) y_i / \sum_i K(x_t - x_i)$ is exactly representable by a pure ConvCNP. However, the GP posterior mean is not representable by any pure ConvCNP (the same factorization barrier as ANPs applies).
\end{proposition}
 
\paragraph{CNN layers enable context--context coupling.}
On a regular grid, each CNN layer with residual connection implements a Toeplitz matrix iteration $\mathbf{r}^{(\ell)} = (\bI + \mathbf{A}_\ell) \mathbf{r}^{(\ell-1)}$, analogous to the TNP update but with Toeplitz structure.
 
\begin{theorem}[ConvCNP Depth for GP Posteriors]
\label{thm:convcnp-gp}
On a regular grid with spacing $\delta$, a ConvCNP with $L$ CNN layers achieves GP posterior approximation error $O(B_k B_y \lambda_{\min}^{-1} ((\sqrt{\kappa}-1)/(\sqrt{\kappa}+1))^L) + O(\delta)$, matching the Chebyshev convergence rate of TNPs.
\end{theorem}
 
A distinctive feature of ConvCNPs is the depth--support tradeoff: on periodic grids, all matrices are circulant and simultaneously diagonalized by the DFT. With unrestricted filter support $p = n$, a single CNN layer suffices (Proposition~\ref{prop:full-support} in Appendix~\ref{app:convcnp-proofs}). With restricted support, the required product $L \cdot \lfloor p/2 \rfloor$ is governed by the trigonometric approximation number of $1/\hat{K}(\omega)$ (Theorem~\ref{thm:depth-support} in Appendix~\ref{app:convcnp-proofs}), recovering the $\sqrt{\kappa}$ rate. 
 
\begin{theorem}[ConvCNPs and ANPs Are Incomparable]
\label{thm:convcnp-anp-incomparable}
$\calF_{\mathrm{ANP}} \not\subseteq \calF_{\mathrm{ConvCNP}}$ (non-stationary kernel smoothers violate translation equivariance) and $\calF_{\mathrm{ConvCNP}} \not\subseteq \calF_{\mathrm{ANP}}$ (ConvCNPs with CNN depth approximate GP posteriors, which are not ANP-representable).
\end{theorem}
 
\section{The Expressiveness Hierarchy}
 
\begin{theorem}[Expressiveness Hierarchy]
\label{thm:hierarchy}
For all $d \geq 1$ and $L \geq 0$:
\begin{enumerate}
    \item[(a)] \textbf{Attention branch:} $\calF_{\mathrm{CNP}}^{(d)} \subsetneq \calF_{\mathrm{ANP}}^{(d)} \subsetneq \calF_{\mathrm{TNP}}^{(1,d)} \subsetneq \calF_{\mathrm{TNP}}^{(2,d)} \subsetneq \cdots$
    \item[(b)] \textbf{Convolutional branch:} $\calF_{\mathrm{CNP}}^{(d)} \subsetneq \calF_{\mathrm{ConvCNP}}^{(0)} \subsetneq \calF_{\mathrm{ConvCNP}}^{(1)} \subsetneq \calF_{\mathrm{ConvCNP}}^{(2)} \subsetneq \cdots$
    \item[(c)] \textbf{Incomparability:} $\calF_{\mathrm{ANP}}^{(d)} \not\subseteq \calF_{\mathrm{ConvCNP}}^{(L)}$ and $\calF_{\mathrm{ConvCNP}}^{(L)} \not\subseteq \calF_{\mathrm{ANP}}^{(d)}$ for all $L, d$.
\end{enumerate}
\end{theorem}
 
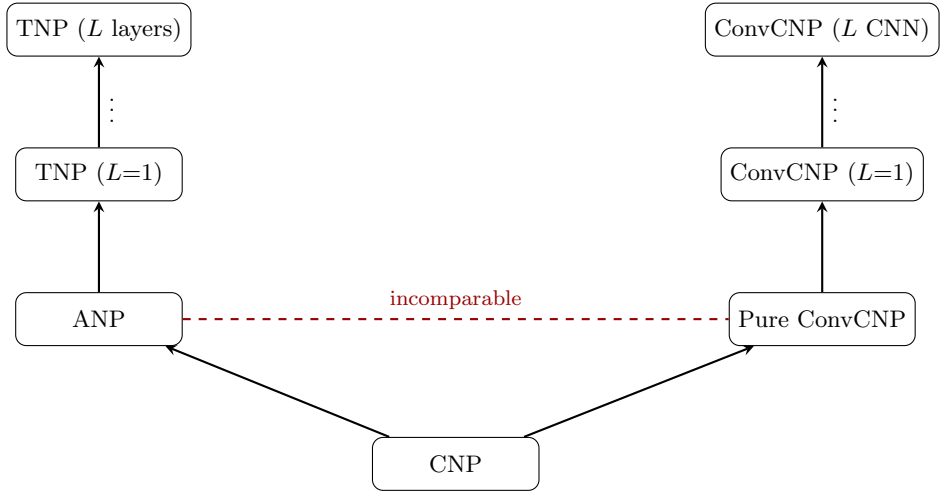
\begin{figure}[htbp]
\centering
\begin{tikzpicture}[
    node distance=1.2cm and 2.5cm,
    every node/.style={draw, rounded corners, minimum width=2.2cm, minimum height=0.7cm, font=\small},
    arrow/.style={->, thick, >=stealth}
]
    \node (cnp) {CNP};
    \node (anp)  [above left=of cnp]  {ANP};
    \node (conv0) [above right=of cnp] {Pure ConvCNP};
    \node (tnp1) [above=of anp] {TNP ($L{=}1$)};
    \node (conv1) [above=of conv0] {ConvCNP ($L{=}1$)};
    \node (tnpL) [above=of tnp1] {TNP ($L$ layers)};
    \node (convL) [above=of conv1] {ConvCNP ($L$ CNN)};
    \draw[arrow] (cnp) -- (anp);
    \draw[arrow] (cnp) -- (conv0);
    \draw[arrow] (anp) -- (tnp1);
    \draw[arrow] (conv0) -- (conv1);
    \draw[arrow] (tnp1) -- (tnpL) node[midway, right, draw=none, minimum width=0pt, minimum height=0pt] {\footnotesize$\vdots$};
    \draw[arrow] (conv1) -- (convL) node[midway, right, draw=none, minimum width=0pt, minimum height=0pt] {\footnotesize$\vdots$};
    \draw[dashed, thick, red!60!black] (anp.east) -- (conv0.west) node[midway, above, draw=none, minimum width=0pt, minimum height=0pt, font=\footnotesize\color{red!60!black}] {incomparable};
\end{tikzpicture}
\caption{The expressiveness hierarchy of Neural Process architectures. Solid arrows denote strict inclusion ($\subsetneq$). The dashed line marks incomparability. For stationary kernels on regular grids, the ConvCNP branch achieves the same Chebyshev convergence rate as the TNP branch.}
\label{fig:hierarchy}
\end{figure}
 
The proof assembles the results from previous sections; the full argument appears in Appendix~\ref{app:hierarchy-proof}.
 
\section{Latent Neural Processes}
\label{sec:latent}
 
Latent NPs augment the deterministic pathway with a global latent $z \sim q(z|C)$, $z \in \R^k$, yielding $p(y_T | X_T, C) = \int p(y_T | X_T, z) q(z|C) dz$. The encoder $q(z|C)$ inherits the limitations of the underlying architecture: if $C \sim_h C'$ for a latent CNP, then $q(z|C) = q(z|C')$ regardless of decoder expressiveness. Thus the impossibility results for CNPs and ANPs lift directly to their latent variants.
 
Assuming an arbitrarily powerful encoder, the constraints of finite latent dimension remain:
 
\begin{theorem}[Latent NP Cannot Represent GP Posterior]
\label{thm:latent-impossibility}
For a Gaussian latent NP with latent dimension $k$ and linear decoder:
(a) Mean matching for all $y_C \in \R^n$ requires $k \geq n$.
(b) Covariance matching at $m$ target points requires $k \geq m$.
(c) Matching for arbitrary target configurations requires $k = \infty$.
\end{theorem}
 
Finite-dimensional latent NPs exactly represent the class of rank-$k$ GPs: processes of the form $f(x) = a(x)^\top z + b(x)$ with $z \sim \calN(m, S)$. For full-rank kernels, truncating the Mercer expansion gives approximation error governed by the spectral tail $\sum_{j>k} \lambda_j$, with rates depending on kernel smoothness (exponential for RBF, polynomial for Mat\'ern). Full proofs appear in Appendix~\ref{app:latent-proofs}.
  
\section{Discussion}
\label{sec:discussion}

Table~\ref{tab:summary} summarizes our results. The analysis reveals that CNPs are limited to functions of finitely many expected features; ANPs extend this to query-dependent reweightings but still factor across context points; only TNPs can capture global structure via self-attention, with tight depth bounds of $\Theta(\sqrt{\kappa}\log(1/\varepsilon))$. ConvCNPs and ANPs are incomparable, separated by stationarity versus translation equivariance.
 
\begin{table}[h]
\centering
\small
\begin{tabular}{lcccc}
\toprule
\textbf{Architecture} & \textbf{Self-Attn/CNN} & \textbf{Dim.\ $d$} & \textbf{Latent $k$} & \textbf{Representable Functions} \\
\midrule
CNP & --- & any & --- & Mean statistics of context \\
ANP & --- & any & --- & Query-reweighted statistics \\
Pure ConvCNP & --- & functional & --- & Stationary kernel smoothers (exact) \\
ConvCNP ($L$ CNN) & $L$ CNN layers & functional & --- & Stationary degree-$L$ kernel poly. \\
TNP ($L$ layers) & $L$ self-attn & $O(n)$ & --- & Degree-$L$ kernel polynomials \\
TNP (deep) & $\Theta(\sqrt{\kappa} \log \frac{1}{\varepsilon})$ & $O(n)$ & --- & $\varepsilon$-approx GP posterior mean \\
\midrule
Latent CNP & --- & any & any & Coherent samples from mean stats \\
Latent ANP & --- & any & any & Coherent samples, query-reweighted \\
Latent TNP & $\Theta(\sqrt{\kappa} \log \frac{1}{\varepsilon})$ & $O(n)$ & $\geq n$ & $\varepsilon$-approx GP posterior dist. \\
\bottomrule
\end{tabular}
\caption{Architecture capabilities. Latent variants add coherent sampling but do not circumvent encoder limitations.}
\label{tab:summary}
\end{table}

Our results provide an answer to a theoretical question open since \citet{garnelo2018cnp} introduced NPs regarding what functions these architectures can represent. The analysis reveals a hierarchy. CNPs with mean aggregation are limited to functions of finitely many expected features of the context distribution which is sufficient for low-dimensional parametric families but unable to capture context-context interactions. ANPs extend this to query-dependent reweightings, enabling kernel smoothers and local adaptation but still factoring across context points. Only TNPs with self-attention can capture the global structure needed for GP posteriors, with depth requirements scaling logarithmically in the kernel matrix condition number.

The CNP impossibility follows from a dimension-counting argument. the set of $h$-equivalent contexts forms an $(n-1)d$-dimensional subspace, so collisions are generic rather than exceptional. The ANP impossibility is more subtle and stems from a factorization constraint because attention weights $\alpha_i \propto f(x_t, x_i, y_i)$ decompose across context points, while GP posterior weights couple all points through $\bK^{-1}$. This coupling-versus-factorization distinction, visible even with just two context points, is the fundamental reason self-attention is necessary. The TNP construction exploits the Neumann series $\bK^{-1} = \sum_{m=0}^\infty (\bI - \bK)^m$, which requires learning an appropriate normalization of the kernel matrix, which is a hidden capacity requirement that may explain why TNPs benefit from careful initialization in practice. Orthogonally to representational capacity, the validity of NP predictions as stochastic processes is governed by conditioning consistency, where the gap for CNPs vanishes as $O(1/n^2)$ in context size \citep{young2026consistency}, while the quantitative cost of amortization decomposes into label contamination, information bottleneck, and encoder sharing terms \citep{young2026costs}, with the bottleneck decay rates matching the spectral tail bounds of \Cref{app:cor:spectral-decay}.

Several limitations warrant discussion. Our analysis characterizes representational capacity, not learnability. A function being TNP-representable does not guarantee that gradient descent will find the right parameters, and the sample complexity of learning within each function class remains open. We also assume idealized attention with exact softmax and infinite precision; practical implementations with learned temperatures, finite precision, and approximate attention kernels may exhibit different effective capacity. For standard kernels (e.g. RBF, Mat\'ern), the condition number $\kappa(\bK)$ typically grows with context size $n$, so the depth bound $\Theta(\sqrt{\kappa}\log(1/\varepsilon))$ implies that fixed-depth TNPs face representational limitations on large contexts. This is an architectural insight, not a limitation of the analysis as it quantifies the cost of approximating GP posteriors as context sets grow.

Convolutional CNPs \citep{Gordon2020Convolutional} replace finite-dimensional aggregation with functional representations, sidestepping the dimension-counting arguments of Section~\ref{sec:cnp}. The picture is more nuanced than a simple placement in the hierarchy. The pure ConvCNP (without CNN) sits at the ANP level as its weights on context points factorize, enabling stationary kernel smoothers but not GP posteriors. However, the full ConvCNP with CNN layers accesses TNP-level expressiveness for stationary kernels on regular grids, achieving the same Chebyshev convergence rate $(\sqrt{\kappa}-1)/(\sqrt{\kappa}+1)$ per layer (Theorem~\ref{thm:convcnp-gp}). The structural difference is that CNN iterations are Toeplitz (constant along diagonals), while TNP attention matrices are unconstrained. This makes ConvCNPs and ANPs incomparable rather than nested as each contains functions outside the other (Theorem~\ref{thm:convcnp-anp-incomparable}). The cost of irregularity remains an open question.

We also offer some thoughts on practicalities. For few-shot regression on parametric families, CNPs should suffice and attention adds complexity without benefit. For image completion and spatial prediction requiring local adaptation, ANPs provide the right inductive bias. For tasks demanding global coherence such as calibrated uncertainty in Bayesian optimization, consistent identity in face completion, exact GP emulation, TNPs are necessary. For spatially structured tasks with stationary kernels and approximately regular observation grids, which is common in environmental monitoring and climate modeling, ConvCNPs offer an attractive middle ground with TNP-level posterior approximation with the parameter efficiency and equivariance guarantees of convolutional architectures.

\clearpage



\bibliography{references}

\clearpage

\appendix

\section{CNP Proofs}
\label{app:cnp-proofs}

\begin{proof}[Proof of Theorem~\ref{thm:cnp-char} (CNP Characterization)]
$(\Rightarrow)$ If $F \in \calF_{\mathrm{moment}}^{(d)}$, then $F(C, x_t) = \phi(\frac{1}{n}\sum_i \psi(x_i, y_i), x_t)$. Setting $h = \psi$ and $g = \phi$ gives exact representation.

$(\Leftarrow)$ If $F$ is $d$-representable, then $F(C, x_t) = g(\bar{h}_C, x_t)$ for some $h, g$. This is exactly the form of $\calF_{\mathrm{moment}}^{(d)}$ with $\psi = h$ and $\phi = g$.
\end{proof}

\begin{proof}[Proof of Proposition~\ref{prop:collisions} (Existence of Collisions)]
Consider the map $\Phi: (\calX \times \calY)^n \to \R^d$ defined by $\Phi(C) = \frac{1}{n} \sum_{i=1}^n h(x_i, y_i)$. The domain has dimension $n(d_x + d_y)$; the codomain has dimension $d$. When $n(d_x + d_y) > d$, the map cannot be injective, so distinct context sets can have identical representations.

More explicitly, for fixed $h$-images $(h_1, \ldots, h_n) \in (\R^d)^n$, perturbations $(\delta_1, \ldots, \delta_n)$ satisfying $\sum_i \delta_i = 0$ form a subspace of dimension $(n-1)d$. Since $h$ is continuous and maps from a higher-dimensional space when $n(d_x + d_y) > d$, generic perturbations in $(\calX \times \calY)^n$ induce such mean-preserving perturbations in the representation space.
\end{proof}

\begin{proof}[Proof of Theorem~\ref{thm:gp-not-cnp} (GP Posterior is Not Finitely Representable)]
The GP posterior mean depends on the inverse Gram matrix $\bK^{-1}$, which couples all context points. The weight on $y_i$ in the posterior mean depends on all pairwise kernel values $k(x_i, x_j)$ for $j \neq i$.

By \Cref{prop:collisions}, for any encoder $h$, there exist distinct context sets $C, C'$ with $\bar{h}_C = \bar{h}_{C'}$. We construct $C, C'$ with identical mean encodings but different x-configurations, hence different Gram matrices $\bK(X_C, X_C) \neq \bK(X_{C'}, X_{C'})$.

Specifically, with $n = d+1$ context points, let $C$ have points $\{x_1, \ldots, x_{d+1}\}$ in general position and $C'$ have points $\{x_1', \ldots, x_{d+1}'\}$ arranged so that $\bar{h}_C = \bar{h}_{C'}$ (possible by the null space argument) but the pairwise distances differ.

For the same target $x_t$ and y-values, the GP posterior means differ:
\[
\mu(x_t | C) = k(x_t, X_C) \bK(X_C, X_C)^{-1} y_C \neq k(x_t, X_{C'}) \bK(X_{C'}, X_{C'})^{-1} y_C = \mu(x_t | C').
\]
Thus no CNP with representation dimension $d$ can exactly represent the GP posterior on contexts of size $n > d$.
\end{proof}

\begin{proof}[Proof of Theorem~\ref{thm:cnp-lower-bound} (CNP Approximation Lower Bound)]
We proceed in four steps.

\textit{Step 1: Reduction to linear encoders.}
Since the target $\mu(x_t | C) = k(x_t, X_C) \bK^{-1} y_C$ is linear in $y_C$ and $y_C$ is Gaussian, the optimal encoder is linear. This follows from a standard result in Gaussian rate-distortion theory: for jointly Gaussian $(y_C, T(y_C))$ with $T$ linear, the minimum MSE estimator of $T(y_C)$ given any function $\phi(y_C)$ compressed to $d$ dimensions is achieved by a linear $\phi$ (see \citet{Berger1971} for example).

For fixed context locations, a CNP encoder $h: \calX \times \calY \to \R^d$ reduces to $\bar{h}_C = \frac{1}{n} \sum_{i=1}^n \phi_i(y_i)$ for functions $\phi_i: \R \to \R^d$. Linearity gives $\phi_i(y_i) = a_i y_i$ for some $a_i \in \R^d$, so $\bar{h}_C = A y_C$ for a matrix $A \in \R^{d \times n}$.

\textit{Step 2: Whitening.}
Let $z = \bK^{-1/2} y_C \sim \calN(0, \bI_n)$. Define:
\begin{align*}
\tilde{w}_{x_t} &= \bK^{-1/2} k(x_t, X_C) \in \R^n \\
\tilde{A} &= A \bK^{1/2} \in \R^{d \times n}
\end{align*}
Then $\mu(x_t | C) = \tilde{w}_{x_t}^\top z$ and $\bar{h}_C = \tilde{A} z$.

\textit{Step 3: Optimal reconstruction.}
For Gaussian $z$, the optimal predictor of $\tilde{w}_{x_t}^\top z$ given $\tilde{A} z$ is:
\[
\E[\tilde{w}_{x_t}^\top z | \tilde{A} z] = \tilde{w}_{x_t}^\top P_{\tilde{A}} z
\]
where $P_{\tilde{A}} = \tilde{A}^\top (\tilde{A} \tilde{A}^\top)^{-1} \tilde{A}$ is the orthogonal projection onto the $d$-dimensional row space of $\tilde{A}$. The MSE at target $x_t$ is:
\[
\E[|\tilde{w}_{x_t}^\top z - \tilde{w}_{x_t}^\top P_{\tilde{A}} z|^2] = \|(\bI - P_{\tilde{A}}) \tilde{w}_{x_t}\|^2.
\]

\textit{Step 4: Integrated error.}
Define the whitened weight matrix $\tilde{W} \in \R^{n \times m}$ with columns $\tilde{w}_{x_1^*}, \ldots, \tilde{w}_{x_m^*}$ for target points $x_1^*, \ldots, x_m^* \sim \nu$. The total MSE is:
\[
\sum_{j=1}^m \|(\bI - P_{\tilde{A}}) \tilde{w}_{x_j^*}\|^2 = \|(\bI - P_{\tilde{A}}) \tilde{W}\|_F^2.
\]
This is minimized when $P_{\tilde{A}}$ projects onto the top $d$ left singular vectors of $\tilde{W}$. If $\tilde{W}$ has singular values $\sigma_1 \geq \cdots \geq \sigma_n$ (with $\sigma_i = 0$ for $i > \rank(\tilde{W})$), the minimum error is $\sum_{i=d+1}^n \sigma_i^2$ and the total variance is $\sum_{i=1}^n \sigma_i^2$.

Under the isotropic second moment assumption:
\[
\frac{1}{m} \tilde{W} \tilde{W}^\top = \frac{1}{m} \sum_{j=1}^m \tilde{w}_{x_j^*} \tilde{w}_{x_j^*}^\top \to \E_{x_t \sim \nu}[\tilde{w}_{x_t} \tilde{w}_{x_t}^\top] = \alpha \bI_n
\]
for some $\alpha > 0$. Thus the singular values of $\tilde{W}$ satisfy $\sigma_i^2 \approx \alpha m / n$ for all $i \leq n$, giving:
\[
\frac{\sum_{i=d+1}^n \sigma_i^2}{\sum_{i=1}^n \sigma_i^2} = \frac{n - d}{n} = 1 - \frac{d}{n}. \qedhere
\]
\end{proof}

We also record the linear regression example, which illustrates the representation requirements concretely.

\begin{proposition}[Linear Regression Requires $d = O(k^2)$]
\label{prop:linear-regression}
Let $\calF_{\mathrm{linear}}^{(k)}$ be the class of linear predictors with $k$-dimensional features $\psi: \calX \to \R^k$:
\[
F(C, x_t) = \langle \beta(C), \psi(x_t) \rangle, \quad \beta(C) = \arg\min_\beta \sum_{i=1}^n (y_i - \langle \beta, \psi(x_i) \rangle)^2.
\]
Then:
\begin{enumerate}
    \item[(a)] $d = \frac{k(k+3)}{2}$ suffices: $\calF_{\mathrm{linear}}^{(k)} \subseteq \calF_{\mathrm{moment}}^{(k(k+3)/2)}$.
    \item[(b)] $d = \Omega(k^2)$ is necessary for exact representation.
\end{enumerate}
\end{proposition}

\begin{proof}
(a) The OLS solution is $\beta(C) = (\sum_i \psi(x_i)\psi(x_i)^\top)^{-1} (\sum_i y_i \psi(x_i))$. Define the encoder:
\[
h(x, y) = (\vech(\psi(x)\psi(x)^\top), y \cdot \psi(x)) \in \R^{k(k+1)/2 + k},
\]
Then $\bar{h}_C = (\frac{1}{n}\mathrm{vec}(\sum_i \psi(x_i)\psi(x_i)^\top), \frac{1}{n}\sum_i y_i \psi(x_i))$. The decoder can recover $\beta(C)$ by inverting the (rescaled) Gram matrix and multiplying.

(b) The Gram matrix $\sum_i \psi(x_i)\psi(x_i)^\top$ has $k(k+1)/2$ degrees of freedom (symmetric). Different Gram matrices generically yield different predictors. Thus $d \geq k(k+1)/2 = \Omega(k^2)$.
\end{proof}

\section{ANP Proofs}
\label{app:anp-proofs}

\begin{proof}[Proof of Theorem~\ref{thm:anp-kernel} (ANPs Represent Kernel Smoothers)]
\label{app:anp-kernel}
Set:
\begin{itemize}
    \item Value: $v(x, y) = (y, 1) \in \R^{d_y + 1}$
    \item Key: $k(x, y) = \psi(x)$ (independent of $y$)
    \item Query: $q(x_t) = \phi(x_t)$
\end{itemize}
Choose $\phi, \psi$ such that $q(x_t)^\top k(x_i, y_i) = \log K(x_t, x_i)$. If $\log K$ admits a finite inner product decomposition, this is achieved exactly. For general continuous $K$ on a compact domain, universal approximation \citep{Hornik1989Universal} gives networks $\phi, \psi$ with $|q(x_t)^\top k(x_i, y_i) - \log K(x_t, x_i)| \leq \delta$ for any $\delta > 0$, yielding $\varepsilon(\delta)$-approximation of the kernel smoother with $\varepsilon(\delta) \to 0$ as $\delta \to 0$.

In the exact case, the attention weights are:
\[
\alpha_i(x_t; C) = \frac{\exp(\log K(x_t, x_i))}{\sum_j \exp(\log K(x_t, x_j))} = \frac{K(x_t, x_i)}{\sum_j K(x_t, x_j)}.
\]

The representation is:
\[
r_C(x_t) = \sum_i \alpha_i(x_t; C) \, (y_i, 1) = \left( \frac{\sum_i K(x_t, x_i) y_i}{\sum_i K(x_t, x_i)}, 1 \right).
\]

A linear decoder extracts the first component, which is exactly the kernel smoother.
\end{proof}

\begin{proof}[Proof of Theorem~\ref{thm:anp-char} (ANP Characterization)]
\label{app:anp-char}
$(\Rightarrow)$ If $F$ is ANP-representable, the attention mechanism gives exactly this form with $s(x_t, x, y) = q(x_t)^\top k(x, y) / \tau$.

$(\Leftarrow)$ Any such $F$ can be implemented by setting $q(x_t)^\top k(x, y) = s(x_t, x, y) \cdot \tau$ and using an appropriate decoder $G$.
\end{proof}

\begin{proof}[Proof of Theorem~\ref{thm:gp-not-anp} (GP Posterior Requires Context-Context Coupling)]
\label{app:gp-not-anp}
The GP posterior mean is:
\[
\mu(x_t | C) = k(x_t, X_C) \, \bK(X_C, X_C)^{-1} \, y_C = \sum_{i=1}^n w_i(x_t; C) \, y_i
\]
where $w_i(x_t; C) = [k(x_t, X_C) \bK^{-1}]_i$.

The weight $w_i$ depends on $\bK^{-1}$, which couples all context points. Specifically, $w_i$ depends on $k(x_j, x_m)$ for $j, m \neq i$.

ANP attention weights factor as:
\[
\alpha_i(x_t; C) = \frac{f(x_t, x_i, y_i)}{\sum_j f(x_t, x_j, y_j)}
\]
for some function $f$ depending on query-key pairs independently. The weight on point $i$ depends only on $(x_t, x_i, y_i)$ and the normalizing constant, not on the relationships between other context points.

\textbf{Explicit counterexample:} Consider $n = 2$ context points. The GP weight on point 1 is:
\[
w_1 = \frac{k(x_t, x_1) k(x_2, x_2) - k(x_t, x_2) k(x_1, x_2)}{k(x_1, x_1) k(x_2, x_2) - k(x_1, x_2)^2}.
\]
This depends on $k(x_1, x_2)$, the kernel value between the two context points. No ANP attention weight can capture this, as $\alpha_1 \propto \exp(q(x_t)^\top k(x_1, y_1))$ involves only the query-point-1 relationship.
\end{proof}

\section{TNP Assumptions and Upper Bound Proofs}
\label{app:tnp-assumptions}

\subsection{Structural Assumptions}

\begin{assumption}[Position-Based Self-Attention]
\label{app:ass:position-attention}
Each self-attention layer uses attention weights that depend only on input positions:
\[
\beta_{ij}^{(\ell)} = \frac{\exp(q_s(x_i)^\top k_s(x_j) / \tau)}{\sum_{m=1}^n \exp(q_s(x_i)^\top k_s(x_m) / \tau)}
\]
where $q_s, k_s: \calX \to \R^{d_k}$ are position encoders. We denote the resulting attention matrix by $\tilde{\bK} \in \R^{n \times n}$, with $[\tilde{\bK}]_{ij} = \beta_{ij}$.
\end{assumption}

\begin{remark}
This assumption can be realized by: (i) using separate position-based attention heads, (ii) concatenating fixed positional encodings that dominate learned representations, or (iii) architectural modifications that enforce position-based routing. The assumption decouples the ``routing'' structure (which context points attend to which) from the ``content'' being routed (the value vectors), enabling our polynomial analysis.

We impose position-based attention to enable clean polynomial analysis of the layer-wise updates. This assumption is relaxed in Appendix~\ref{app:tnp-lower}, where we show that the depth lower bound holds even for fully adaptive representation-based attention.
\end{remark}

\begin{assumption}[Residual Connections]
\label{app:ass:residual}
Self-attention layers use residual connections with learnable value projections:
\[
h_i^{(\ell)} = h_i^{(\ell-1)} + \sum_{j=1}^n \beta_{ij}^{(\ell)} W_v^{(\ell)} h_j^{(\ell-1)}
\]
where $W_v^{(\ell)} \in \R^{d \times d}$ is the value projection matrix at layer $\ell$.
\end{assumption}

\begin{assumption}[Kernel Matrix Conditioning]
\label{app:ass:conditioning}
The kernel $k$ is positive definite, and for context sets of size $n$, the Gram matrix $\bK$ has eigenvalues $0 < \lambda_{\min} \leq \lambda_1 \leq \cdots \leq \lambda_n \leq \lambda_{\max}$. The condition number is $\kappa = \lambda_{\max}/\lambda_{\min}$.
\end{assumption}

\begin{assumption}[Attention Approximates Normalized Kernel]
\label{app:ass:attention-kernel}
The attention matrix $\tilde{\bK}$ satisfies $\tilde{\bK} = D^{-1} \bK$ where $D = \diag(\bK \mathbf{1})$ is the row-sum normalization. Realistically, the TNP must learn position encoders $q_s, k_s$ achieving this approximation, which is a nontrivial learning problem.
\end{assumption}

\subsection{Matrix Form and Basic Lemmas}

\begin{lemma}[Matrix Form of Updates]
\label{app:lem:matrix-form}
Under Assumptions~\ref{app:ass:position-attention} and \ref{app:ass:residual}, the layer-$\ell$ update is:
\[
H^{(\ell)} = H^{(\ell-1)} + \tilde{\bK} H^{(\ell-1)} W_v^{(\ell)\top}
\]
or in vectorized form:
\[
\mathrm{vec}(H^{(\ell)}) = \left( \bI_{nd} + W_v^{(\ell)} \otimes \tilde{\bK} \right) \mathrm{vec}(H^{(\ell-1)})
\]
where $\otimes$ denotes the Kronecker product. For scalar value weights $W_v^{(\ell)} = \alpha_\ell \bI_d$:
\[
H^{(\ell)} = (\bI_n + \alpha_\ell \tilde{\bK}) H^{(\ell-1)}.
\]
\end{lemma}

\begin{lemma}[One Layer Captures Kernel-Weighted Sums]
\label{app:lem:one-layer}
Under Assumption~\ref{app:ass:position-attention}, if the position encoders $q_s, k_s$ satisfy $q_s(x)^\top k_s(x') = \tau \log K(x, x') + c$ for some kernel $K$ and constant $c$, then:
\[
\beta_{ij} = \frac{K(x_i, x_j)}{\sum_{m=1}^n K(x_i, x_m)}.
\]
After one self-attention layer with value function $v: \calX \times \calY \to \R^d$:
\[
h_i^{(1)} = h_i^{(0)} + \sum_{j=1}^n \frac{K(x_i, x_j)}{\sum_m K(x_i, x_m)} v(x_j, y_j).
\]
\end{lemma}

\begin{proof}
Direct computation:
\[
\beta_{ij} = \frac{\exp(q_s(x_i)^\top k_s(x_j)/\tau)}{\sum_m \exp(q_s(x_i)^\top k_s(x_m)/\tau)} = \frac{\exp(\log K(x_i, x_j) + c)}{\sum_m \exp(\log K(x_i, x_m) + c)} = \frac{K(x_i, x_j)}{\sum_m K(x_i, x_m)}.
\]
The representation update follows from the definition.
\end{proof}

\begin{proposition}[Gram Matrix Encoding]
\label{app:prop:gram-encoding}
With representation dimension $d \geq n$ and initial encoding $h^{(0)}(x_j, y_j) = (e_j, y_j) \in \R^{n + d_y}$ where $e_j$ is the $j$-th standard basis vector, one self-attention layer can produce representations containing the $i$-th row of $\tilde{\bK}$:
\[
h_i^{(1)} = (e_i + \tilde{\bK}_{i \cdot}, y_i + \text{(weighted sum of } y_j\text{)})
\]
where $\tilde{\bK}_{i\cdot} = (\beta_{i1}, \ldots, \beta_{in})$ is the $i$-th row of the attention matrix.
\end{proposition}

\begin{proof}
Set $W_v^{(1)} = \bI$. Then $h_i^{(1)} = h_i^{(0)} + \sum_j \beta_{ij} h_j^{(0)} = (e_i, y_i) + \sum_j \beta_{ij} (e_j, y_j) = (e_i + \tilde{\bK}_{i\cdot}, y_i + \sum_j \beta_{ij} y_j)$.
\end{proof}

\subsection{Polynomial Computation}

\begin{proof}[Proof of Theorem~\ref{thm:polynomial} (Polynomial Representation)]
By induction. For $L = 1$: $H^{(1)} = (\bI + \alpha_1 \tilde{\bK}) H^{(0)}$. Assuming the result for $L-1$:
\[
H^{(L)} = (\bI + \alpha_L \tilde{\bK}) H^{(L-1)} = (\bI + \alpha_L \tilde{\bK}) \prod_{\ell=1}^{L-1} (\bI + \alpha_\ell \tilde{\bK}) H^{(0)} = \prod_{\ell=1}^{L} (\bI + \alpha_\ell \tilde{\bK}) H^{(0)}.
\]
Expanding the product, the coefficient of $\tilde{\bK}^m$ is the sum over all ways to choose $m$ factors to contribute $\alpha_\ell \tilde{\bK}$ (and the remaining $L-m$ factors to contribute $\bI$), which is exactly $e_m(\alpha_1, \ldots, \alpha_L)$.
\end{proof}

\begin{corollary}[Achievable Coefficient Space]
\label{app:cor:achievable}
The set of achievable coefficient vectors $(c_0, c_1, \ldots, c_L) \in \R^{L+1}$ for polynomials $\sum_{m=0}^L c_m \tilde{\bK}^m$ representable by an $L$-layer TNP with scalar value weights is:
\[
\mathcal{A}_L = \{(e_0(\alpha), e_1(\alpha), \ldots, e_L(\alpha)) : \alpha \in \R^L\}
\]
where $e_0 \equiv 1$. This is an $L$-dimensional algebraic variety in $\R^{L+1}$, not all of $\R^{L+1}$.
\end{corollary}

\begin{proof}
The map $\alpha \mapsto (e_0(\alpha), \ldots, e_L(\alpha))$ has image determined by Newton's identities relating elementary symmetric polynomials to power sums. Since $e_0 = 1$ always, the image lies in the hyperplane $\{c_0 = 1\}$. Within this hyperplane, the image is the set of coefficient vectors of constant term 1 polynomials with real roots (since $\prod_\ell(1 + \alpha_\ell z) = \sum_m e_m z^m$ has roots $-1/\alpha_\ell$). This is a proper subset of $\R^L$.
\end{proof}

\begin{remark}[Non-Scalar Value Matrices]
With general $W_v^{(\ell)} \in \R^{d \times d}$, the achievable polynomial space is larger but the analysis becomes representation-dependent. For a fixed initial representation $H^{(0)}$ and target polynomial $p(\tilde{\bK})$, the question becomes whether there exist $W_v^{(1)}, \ldots, W_v^{(L)}$ such that $\prod_\ell (\bI + \tilde{\bK} W_v^{(\ell)\top}) H^{(0)} = p(\tilde{\bK}) H^{(0)}$. This is generically solvable for polynomials of degree $\leq L$ when $d$ is sufficiently large, but a complete characterization requires specifying $H^{(0)}$.
\end{remark}

\label{app:tnp-upper}
\subsection{Approximating the Kernel Matrix Inverse}

\begin{lemma}[Neumann Series]
\label{app:lem:neumann}
Let $A \in \R^{n \times n}$ with spectral radius $\rho(A) < 1$. Then $(\bI - A)^{-1} = \sum_{m=0}^{\infty} A^m$, and the truncated series satisfies:
\[
\left\| (\bI - A)^{-1} - \sum_{m=0}^{L-1} A^m \right\| \leq \frac{\rho(A)^L}{1 - \rho(A)}.
\]
\end{lemma}

\begin{proof}
Standard result from matrix analysis \citep{golub2013matrix}. The truncation error is $\|(\bI - A)^{-1} A^L\| = \|(\bI - A)^{-1}\| \|A^L\| \leq \frac{1}{1-\rho(A)} \rho(A)^L$.
\end{proof}

\begin{proposition}[Inverse via Neumann Series]
\label{app:prop:neumann-inverse}
Under Assumption~\ref{app:ass:conditioning}, setting $A = \bI - \bK/\lambda_{\max}$:
\begin{enumerate}
    \item[(a)] The eigenvalues of $A$ lie in $[0, 1 - 1/\kappa]$, so $\rho(A) = 1 - 1/\kappa < 1$.
    \item[(b)] $\bK^{-1} = \frac{1}{\lambda_{\max}} (\bI - A)^{-1} = \frac{1}{\lambda_{\max}} \sum_{m=0}^{\infty} A^m$.
    \item[(c)] The truncated series gives:
    \[
    \left\| \bK^{-1} - \frac{1}{\lambda_{\max}} \sum_{m=0}^{L-1} A^m \right\| \leq \frac{\rho^L}{\lambda_{\min}}, \quad \rho = 1 - \frac{1}{\kappa}.
    \]
\end{enumerate}
\end{proposition}

\begin{proof}
(a) If $\bK v = \lambda v$, then $Av = (1 - \lambda/\lambda_{\max})v$. Since $\lambda \in [\lambda_{\min}, \lambda_{\max}]$, the eigenvalues of $A$ lie in $[0, 1 - \lambda_{\min}/\lambda_{\max}] = [0, 1 - 1/\kappa]$.

(b) $\bK = \lambda_{\max}(\bI - A)$, so $\bK^{-1} = \frac{1}{\lambda_{\max}}(\bI - A)^{-1}$.

(c) By Lemma~\ref{app:lem:neumann}: $\|(\bI - A)^{-1} - \sum_{m=0}^{L-1} A^m\| \leq \frac{\rho^L}{1 - \rho} = \frac{\rho^L}{1/\kappa} = \kappa \rho^L$. Dividing by $\lambda_{\max}$: $\|\bK^{-1} - \frac{1}{\lambda_{\max}}\sum_{m=0}^{L-1} A^m\| \leq \frac{\kappa \rho^L}{\lambda_{\max}} = \frac{\rho^L}{\lambda_{\min}}$.
\end{proof}

\begin{proposition}[Idealized TNP Approximation of GP Posterior]
\label{app:thm:tnp-gp-approx}
Suppose a TNP architecture has access to:
\begin{enumerate}
    \item[(i)] Position-based self-attention with attention matrix equal to the row-normalized kernel $\tilde{\bK} = D^{-1}\bK$,
    \item[(ii)] The normalization constants $D_{ii} = \sum_j K(x_i, x_j)$ via the encoding,
    \item[(iii)] Scalar value weights $\alpha_1, \ldots, \alpha_L$ implementing the truncated Neumann series.
\end{enumerate}
Then the GP posterior mean can be approximated with error $O(\rho^L / \lambda_{\min})$ where $\rho = 1 - 1/\kappa$.

This is an existence result. It establishes that the approximation lies within the representational capacity of TNPs, but does not address whether gradient-based learning can find the required parameters.
\end{proposition}

\begin{proof}
By Proposition~\ref{app:prop:neumann-inverse}, the truncated Neumann series $\hat{\bK}^{-1} = \frac{1}{\lambda_{\max}} \sum_{m=0}^{L-1} (\bI - \bK/\lambda_{\max})^m$ satisfies $\|\bK^{-1} - \hat{\bK}^{-1}\| \leq \rho^L / \lambda_{\min}$.

The approximation error in the posterior mean is:
\[
|\mu_{\mathrm{TNP}} - \mu| = |k(x_t, X_C)(\hat{\bK}^{-1} - \bK^{-1}) y_C| \leq \|k(x_t, X_C)\| \cdot \|\hat{\bK}^{-1} - \bK^{-1}\| \cdot \|y_C\|. \qedhere
\]
\end{proof}

\begin{corollary}[Depth Requirement]
\label{app:cor:depth}
To achieve $\varepsilon$-approximation of the GP posterior mean uniformly over targets $x_t$ with $\|k(x_t, X_C)\| \leq B_k$ and observations $\|y_C\| \leq B_y$, a TNP requires:
\[
L \geq \frac{\log(B_k B_y / (\varepsilon \lambda_{\min}))}{\log(1/\rho)} = O\left(\kappa \log \frac{B_k B_y}{\varepsilon \lambda_{\min}}\right)
\]
self-attention layers, using $\log(1/\rho) \approx 1/\kappa$ for large $\kappa$.
\end{corollary}

\begin{proof}[Proof of Proposition~\ref{prop:chebyshev-upper} (Chebyshev Upper Bound)]
Set
\[
\alpha_\ell = -\frac{2}{\lambda_{\max} + \lambda_{\min} + (\lambda_{\max} - \lambda_{\min})\cos\theta_\ell}, \quad \theta_\ell = \frac{(2\ell-1)\pi}{2L}.
\]
These are the optimal Chebyshev iteration parameters for inverting a matrix with spectrum in $[\lambda_{\min}, \lambda_{\max}]$. Each $\alpha_\ell$ is real and negative with $\alpha_\ell \in (-2/\lambda_{\min}, -2/\lambda_{\max})$, so the polynomial $p_L(\lambda) = \prod_\ell(1 + \alpha_\ell \lambda)$ satisfies $p_L(0) = 1$ and has all roots $-1/\alpha_\ell$ in $[\lambda_{\min}, \lambda_{\max}]$.

The residual polynomial $r_L(\lambda) = 1 - \lambda \cdot p_L(\lambda)$ is the degree-$L$ polynomial vanishing at $0$ that deviates least from $1$ on $[\lambda_{\min}, \lambda_{\max}]$ in the sup-norm. Standard Chebyshev theory \citep{trefethen2019approximation} gives $\|r_L\|_\infty \leq 2\rho^L$ where $\rho = (\sqrt{\kappa}-1)/(\sqrt{\kappa}+1)$. The stated bound follows from $\|p_L(\bK) - \bK^{-1}\| = \|\bK^{-1}\| \cdot \|r_L(\bK)\| \leq \lambda_{\min}^{-1} \cdot 2\rho^L$.
\end{proof}

\section{TNP Lower Bound Proofs}
\label{app:tnp-lower}

\subsection{Linearization}

The GP posterior mean $\mu(x_t | C) = k(x_t, X_C) \bK^{-1} y_C$ is linear in $y_C$. Any TNP approximating this target must itself be approximately linear.

\begin{lemma}[Linearization]
\label{lem:linearization}
Let $F: \calC \times \calX \to \R$ be a TNP achieving
\[
\sup_{\|y_C\| \leq 1} |F(C, x_t) - k(x_t, X_C) \bK^{-1} y_C| \leq \varepsilon.
\]
Assume $F$ is twice continuously differentiable in $y_C$ with $\|\nabla^2_{y_C} F\| \leq L_2$ uniformly on $\|y_C\| \leq 1$. Then:
\[
\left\| \frac{\partial F}{\partial y_C}\bigg|_{y_C = 0} - k(x_t, X_C) \bK^{-1} \right\| \leq 2\varepsilon + \frac{L_2}{2}.
\]
\end{lemma}

\begin{proof}
Let $M = \frac{\partial F}{\partial y_C}\big|_{y_C = 0}$ and $T = k(x_t, X_C) \bK^{-1}$. For any unit vector $u$, Taylor expansion gives:
\[
F(\delta u) = F(0) + \delta \cdot Mu + R(\delta)
\]
where $|R(\delta)| \leq \frac{L_2}{2} \delta^2$ by the smoothness assumption.

The target satisfies $T(\delta u) := k(x_t, X_C) \bK^{-1} (\delta u) = \delta \cdot Tu$ (exactly linear).

By the approximation hypothesis:
\[
|F(\delta u) - \delta \cdot Tu| \leq \varepsilon \quad \text{for all } |\delta| \leq 1.
\]

Evaluating at $\delta = 0$: $|F(0)| \leq \varepsilon$.

For $\delta \neq 0$:
\[
|F(0) + \delta \cdot Mu + R(\delta) - \delta \cdot Tu| \leq \varepsilon
\]
\[
|\delta| \cdot |Mu - Tu| \leq \varepsilon + |F(0)| + |R(\delta)| \leq 2\varepsilon + \frac{L_2}{2}\delta^2.
\]

Setting $\delta = 1$:
\[
|Mu - Tu| \leq 2\varepsilon + \frac{L_2}{2}.
\]

Since $u$ was an arbitrary unit vector, $\|M - T\| \leq 2\varepsilon + \frac{L_2}{2}$.
\end{proof}

\begin{remark}[Smoothness of TNPs]
\label{app:rem:tnp-smoothness}
The smoothness assumption holds for TNPs with standard components. Specifically:
\begin{enumerate}
    \item Softmax attention: $\beta_{ij} = \exp(s_{ij})/\sum_k \exp(s_{ik})$ is $C^\infty$ with derivatives bounded by functions of the scores.
    \item MLP layers with smooth activations (e.g.\ GELU, softplus, or sufficiently smooth approximations to ReLU) have bounded second derivatives on compact domains.
    \item Compositions of $C^2$ functions with bounded derivatives remain $C^2$ with bounded derivatives.
\end{enumerate}
For a TNP with $L$ layers, representation dimension $d$, and weights bounded by $W_{\max}$, the constant $L_2$ scales polynomially in these quantities. For the asymptotic depth bounds, we treat $L_2$ as a fixed constant depending on the architecture but not on $\varepsilon$ or $\kappa$.
\end{remark}

\subsection{Jacobian Evolution through Self-Attention}

\begin{definition}[Jacobian Matrix]
At layer $\ell$, define the Jacobian matrix $B^{(\ell)} \in \R^{n \times n}$ (suppressing the representation dimension $d$) by
\[
B^{(\ell)}_{im} = \frac{\partial h_i^{(\ell)}}{\partial y_m} \bigg|_{y_C = 0}.
\]
\end{definition}

\begin{lemma}[Initial Jacobian]
\label{app:lem:initial-jacobian}
For an encoder $h^{(0)}(x_i, y_i) = a(x_i) + b(x_i) y_i + O(y_i^2)$, the initial Jacobian is diagonal:
\[
B^{(0)}_{im} = \delta_{im} b(x_i).
\]
\end{lemma}

\begin{theorem}[Jacobian Evolution]
\label{app:thm:jacobian-evolution}
Consider a self-attention layer with representation-based attention:
\[
h_i^{(\ell)} = h_i^{(\ell-1)} + \sum_{j=1}^n \beta_{ij}^{(\ell)} W_v h_j^{(\ell-1)}
\]
where $\beta_{ij}^{(\ell)} = \mathrm{softmax}_j(q(h_i^{(\ell-1)})^\top k(h_j^{(\ell-1)}) / \tau)$.

Let $\tilde{\bK} = \beta^{(\ell)}|_{y_C = 0}$ be the attention matrix evaluated at $y_C = 0$. Then:
\[
B^{(\ell)} = B^{(\ell-1)} + \tilde{\bK} W_v B^{(\ell-1)} + \Gamma^{(\ell)} B^{(\ell-1)}
\]
where $\Gamma^{(\ell)}$ is a matrix depending on $X$ and network parameters, arising from attention gradients.
\end{theorem}

\begin{proof}
Differentiating the layer update with respect to $y_m$ and evaluating at $y_C = 0$:
\[
\frac{\partial h_i^{(\ell)}}{\partial y_m} = \frac{\partial h_i^{(\ell-1)}}{\partial y_m} + \sum_j \frac{\partial \beta_{ij}^{(\ell)}}{\partial y_m} W_v h_j^{(\ell-1)} + \sum_j \beta_{ij}^{(\ell)} W_v \frac{\partial h_j^{(\ell-1)}}{\partial y_m}.
\]

At $y_C = 0$, the third term becomes $\sum_j \tilde{\bK}_{ij} W_v B^{(\ell-1)}_{jm}$.

For the attention gradient term, we use the softmax derivative:
\[
\frac{\partial \beta_{ij}}{\partial y_m} = \beta_{ij} \left( \frac{\partial s_{ij}}{\partial y_m} - \sum_k \beta_{ik} \frac{\partial s_{ik}}{\partial y_m} \right)
\]
where $s_{ij} = q(h_i)^\top k(h_j) / \tau$.

The score gradient is:
\[
\frac{\partial s_{ij}}{\partial y_m} = \frac{1}{\tau} \left[ (\nabla_h q|_{h_i})^\top \frac{\partial h_i}{\partial y_m} \cdot k(h_j) + q(h_i) \cdot (\nabla_h k|_{h_j})^\top \frac{\partial h_j}{\partial y_m} \right].
\]

At $y_C = 0$, $\frac{\partial h_i^{(0)}}{\partial y_m} = \delta_{im} b(x_i)$, so $\frac{\partial s_{ij}}{\partial y_m}\big|_{y_C=0}$ is nonzero only when $m \in \{i, j\}$.

Collecting terms, the attention gradient contribution has the form $\Gamma^{(\ell)} B^{(\ell-1)}$ where $\Gamma^{(\ell)}$ depends on $\tilde{\bK}$, query/key gradients, and initial representations---all functions of $X$ alone.
\end{proof}

\begin{corollary}[Polynomial Structure]
\label{cor:polynomial-structure}
After $L$ self-attention layers, the Jacobian $B^{(L)} = \frac{\partial H^{(L)}}{\partial y_C}\big|_{y_C=0}$ has the form:
\[
B^{(L)} = \sum_{\substack{m_1, \ldots, m_k \geq 0 \\ m_1 + \cdots + m_k \leq L}} C_0 \tilde{\bK}^{m_1} C_1 \tilde{\bK}^{m_2} \cdots C_{k-1} \tilde{\bK}^{m_k} C_k
\]
where each $C_i \in \R^{n \times n}$ depends on context locations $X$ and network parameters, but not on observations $y_C$. 

In particular, $B^{(L)}$ is a matrix polynomial in $\tilde{\bK}$ of degree at most $2L$. The factor of $2$ arises because the attention gradient terms $\Gamma^{(\ell)}$ involve quadratic products $\tilde{\bK}_{ij}\tilde{\bK}_{ik}$, contributing up to degree $2$ per layer when projected onto the eigenvalue-controlled family.
\end{corollary}

\begin{proof}
By induction on $L$. The base case $L = 0$ gives $B^{(0)} = C_0$ where $[C_0]_{im} = \delta_{im} b(x_i)$ is diagonal. 

For the inductive step, Theorem~\ref{app:thm:jacobian-evolution} gives:
\[
B^{(\ell)} = (\bI + \Gamma^{(\ell)}) B^{(\ell-1)} + \tilde{\bK} W_v B^{(\ell-1)}
\]
where $\Gamma^{(\ell)}$ depends only on $X$. If $B^{(\ell-1)}$ is a degree-$(\ell-1)$ polynomial in $\tilde{\bK}$, then $B^{(\ell)}$ is degree at most $\ell$.
\end{proof}

\subsection{Spectral Analysis}

\begin{definition}[Spectral Content]
Let $\{v_1, \ldots, v_n\}$ be the eigenvectors of $\tilde{\bK}$ with eigenvalues $\mu_1 \leq \cdots \leq \mu_n$. For any vector $u \in \R^n$, its spectral content is $c = (c_1, \ldots, c_n)$ where $u = \sum_{j=1}^n c_j v_j$.
\end{definition}

\begin{lemma}[Quadratic Form Structure]
\label{app:lem:quadratic-form}
Let $M = \sum_{m=0}^{L} C_m \tilde{\bK}^m$ be a commutative degree-$L$ matrix polynomial (i.e., $C_m$ commutes with $\tilde{\bK}$ for all $m$). Let $\{v_1, \ldots, v_n\}$ be the orthonormal eigenvectors of $\tilde{\bK}$ with eigenvalues $\mu_1 \leq \cdots \leq \mu_n$. Then for any eigenvector $v_j$:
\[
v_j^\top M v_j = \sum_{m=0}^{L} (v_j^\top C_m v_j) \mu_j^m,
\]
a univariate polynomial in $\mu_j$ of degree at most $L$.

For the general (non-commutative) interleaved form of Corollary~\ref{cor:polynomial-structure}, the analogous reduction to a univariate polynomial in a single eigenvalue requires restricting to kernel families where all but one eigenvalue is constant, as in Lemma~\ref{app:lem:eigenvalue-family}.
\end{lemma}

\begin{proof}
Write $\tilde{\bK} = \sum_{k} \mu_k v_k v_k^\top$. Then $\tilde{\bK}^m = \sum_k \mu_k^m v_k v_k^\top$, so:
\[
u^\top \tilde{\bK}^m u = \sum_k \mu_k^m (u^\top v_k)^2.
\]
For the full polynomial:
\[
u^\top M u = \sum_{m=0}^{L} u^\top C_m \tilde{\bK}^m u = \sum_{m=0}^{L} \sum_{k} \mu_k^m \cdot u^\top C_m v_k \cdot v_k^\top u.
\]
Expanding $C_m v_k = \sum_j (v_j^\top C_m v_k) v_j$:
\[
u^\top M u = \sum_{m=0}^{L} \sum_{j,k} \mu_k^m (v_j^\top C_m v_k)(u^\top v_j)(u^\top v_k).
\]
Setting $Q_{jk}(\mu) = \sum_{m=0}^{L} (v_j^\top C_m v_k) \mu_k^m$ gives the result. Each $Q_{jk}$ has degree at most $L$ in $\mu_k$, hence total degree at most $L$.
\end{proof}

\subsection{Eigenvalue-Controlled Kernel Family}

\begin{lemma}[Eigenvalue-Controlled Kernel Family]
\label{app:lem:eigenvalue-family}
For any $\kappa > 1$ and $n \geq 2$, there exists a family of positive definite 
matrices $\{K_t\}_{t \in [0, 1-1/\kappa]}$ such that:
\begin{enumerate}
    \item[(i)] All row sums equal 1: $K_t \mathbf{1} = \mathbf{1}$, hence $D_t = I$ 
               and $\tilde{K}_t = K_t$
    \item[(ii)] The minimum eigenvalue is $\mu_1(t) = 1/\kappa + t \in [1/\kappa, 1]$
    \item[(iii)] The minimum eigenvector $v_1$ is constant across the family, 
                 with $v_1 \perp \mathbf{1}$
    \item[(iv)] All other eigenvalues equal 1
    \item[(v)] $\alpha = v_1^\top D_t^{-1} v_1 = 1$ for all $t$
\end{enumerate}
\end{lemma}

\begin{proof}
Let $\{v_1, \ldots, v_{n-1}\}$ be any orthonormal set in $\mathbf{1}^\perp$, and set 
$v_n = \mathbf{1}/\sqrt{n}$. Define:
\[
K_0 = \frac{1}{\kappa} v_1 v_1^\top + \sum_{j=2}^{n} v_j v_j^\top 
    = I + \left(\frac{1}{\kappa} - 1\right) v_1 v_1^\top.
\]
Then $K_0$ is positive definite with eigenvalue $1/\kappa$ for $v_1$ and eigenvalue 
$1$ for $v_2, \ldots, v_n$.

Row sums: $K_0 \mathbf{1} = \mathbf{1} + (1/\kappa - 1) v_1 (v_1^\top \mathbf{1}) 
= \mathbf{1}$ since $v_1 \perp \mathbf{1}$.

For $t \in [0, 1 - 1/\kappa]$, set $K_t = K_0 + t \cdot v_1 v_1^\top$. Then:
\begin{itemize}
    \item $K_t \mathbf{1} = K_0 \mathbf{1} + t \cdot v_1 (v_1^\top \mathbf{1}) 
           = \mathbf{1}$
    \item $K_t v_1 = (1/\kappa + t) v_1$
    \item $K_t v_j = v_j$ for $j \geq 2$
\end{itemize}
The minimum eigenvalue $1/\kappa + t$ remains below $1$ for $t \leq 1 - 1/\kappa$, 
so $v_1$ remains the minimum eigenvector throughout.
\end{proof}

\subsection{Reduction to Univariate Approximation}

\begin{lemma}[Reduction to Univariate Approximation]
\label{app:lem:univariate}
Let $\bK_t$ be the family from Lemma~\ref{app:lem:eigenvalue-family} with $t \in [\mu_{\min}, \mu_{\max}]$, and let $M(X)$ be the Jacobian of an $L$-layer TNP evaluated at $y_C = 0$. Then there exists a univariate polynomial $q$ of degree at most $2L$ such that:
\[
v_1^\top M(\bK_t) v_1 = q(t) \quad \text{for all } t \in [\mu_{\min}, \mu_{\max}].
\]
\end{lemma}

\begin{proof}
By Corollary~\ref{cor:polynomial-structure}, $M$ is a sum of terms of the form $C_0 \tilde{\bK}^{m_1} C_1 \cdots C_k$ with $\sum m_i \leq L$, where each $C_i$ depends on $X$ but not on $y_C$.

For the family $\bK_t$ from Lemma~\ref{app:lem:eigenvalue-family}, we have $\tilde{\bK}_t = \bK_t$ (since $D_t = I$) with eigenvalues $\mu_1(t) = 1/\kappa + t$ and $\mu_j = 1$ for $j \geq 2$. The eigenvector $v_1$ is fixed across the family.

Consider any term $C_0 \tilde{\bK}_t^{m_1} C_1 \tilde{\bK}_t^{m_2} \cdots C_k$. Expanding in the eigenbasis $\{v_1, \ldots, v_n\}$:
\[
v_1^\top C_0 \tilde{\bK}_t^{m_1} C_1 \cdots C_k v_1 = \sum_{j_0, \ldots, j_k} (v_1^\top C_0 v_{j_0})(v_{j_0}^\top \tilde{\bK}_t^{m_1} v_{j_1}) \cdots (v_{j_{k-1}}^\top C_k v_1).
\]
Since $\tilde{\bK}_t^{m} v_j = \mu_j^m v_j$, each factor $v_{j_{i-1}}^\top \tilde{\bK}_t^{m_i} v_{j_i} = \mu_{j_i}^{m_i} \delta_{j_{i-1}, j_i}$. The sum collapses to paths through eigenvector indices.

For indices $j \geq 2$, we have $\mu_j = 1$, contributing constant factors. Only paths passing through $j = 1$ contribute powers of $\mu_1(t)$. Therefore:
\[
v_1^\top M(\bK_t) v_1 = \sum_{m=0}^{L} c_m \mu_1(t)^m = q(\mu_1(t))
\]
where the coefficients $c_m$ depend on the matrices $C_i$ (hence on network parameters) but not on $t$. The degree is at most $2L$ since each of the $L$ layers contributes at most degree $2$ due to the quadratic attention gradient terms.
\end{proof}

\subsection{Chebyshev Barrier}

\begin{theorem}[Chebyshev Lower Bound]
\label{app:thm:chebyshev}
For any degree-$L$ polynomial $p$ and any interval $[a, b]$ with $0 < a < b$:
\[
\max_{\mu \in [a,b]} \left| p(\mu) - \frac{1}{\mu} \right| \geq \frac{2}{a + b} \rho^L
\]
where $\rho = \frac{\sqrt{b/a} - 1}{\sqrt{b/a} + 1} = \frac{\sqrt{\kappa} - 1}{\sqrt{\kappa} + 1}$ and $\kappa = b/a$ is the condition number.
\end{theorem}

\begin{proof}
This is a classical result from approximation theory. The optimal degree-$L$ polynomial approximation to $1/\mu$ on $[a,b]$ is achieved by appropriately shifted and scaled Chebyshev polynomials, with the stated error bound. See \cite{trefethen2019approximation}.
\end{proof}

\subsection{Spectral Relationship}

\begin{lemma}[Spectral Relationship]
\label{app:lem:spectral-relationship}
Let $\gamma = d_{\max}/d_{\min}$ where $d_{\max} = \max_i D_{ii}$ and $d_{\min} = \min_i D_{ii}$. Then:
\[
\frac{1}{\gamma} \kappa(\bK) \leq \kappa(\tilde{\bK}) \leq \gamma \cdot \kappa(\bK).
\]
\end{lemma}

\begin{proof}
The matrix $\tilde{\bK} = D^{-1}\bK$ is similar to $\hat{\bK} = D^{-1/2} \bK D^{-1/2}$:
\[
D^{1/2} \tilde{\bK} D^{-1/2} = D^{-1/2} \bK D^{-1/2} = \hat{\bK}.
\]
Thus $\tilde{\bK}$ and $\hat{\bK}$ share eigenvalues, and we analyze the latter via Rayleigh quotients.

\noindent\textbf{Upper bound on $\kappa(\tilde{\bK})$:} For any unit vector $w$, setting $u = D^{-1/2}w$ gives:
\[
\frac{w^\top \hat{\bK} w}{w^\top w} = \frac{u^\top \bK u}{u^\top D u} \in \left[ \frac{\lambda_{\min}(\bK)}{d_{\max}}, \frac{\lambda_{\max}(\bK)}{d_{\min}} \right]
\]
where the bounds follow from $d_{\min} \|u\|^2 \leq u^\top D u \leq d_{\max} \|u\|^2$. Therefore:
\[
\kappa(\tilde{\bK}) \leq \frac{\lambda_{\max}(\bK)/d_{\min}}{\lambda_{\min}(\bK)/d_{\max}} = \gamma \cdot \kappa(\bK).
\]

\noindent\textbf{Lower bound on $\kappa(\tilde{\bK})$:} Let $v_{\max}$ be the unit eigenvector of $\bK$ with eigenvalue $\lambda_{\max}(\bK)$. Evaluating at $w = D^{1/2}v_{\max}$:
\[
\lambda_{\max}(\tilde{\bK}) \geq \frac{v_{\max}^\top \bK v_{\max}}{v_{\max}^\top D v_{\max}} = \frac{\lambda_{\max}(\bK)}{v_{\max}^\top D v_{\max}} \geq \frac{\lambda_{\max}(\bK)}{d_{\max}}.
\]

Let $v_{\min}$ be the unit eigenvector of $\bK$ with eigenvalue $\lambda_{\min}(\bK)$. Evaluating at $w = D^{1/2}v_{\min}$:
\[
\lambda_{\min}(\tilde{\bK}) \leq \frac{v_{\min}^\top \bK v_{\min}}{v_{\min}^\top D v_{\min}} = \frac{\lambda_{\min}(\bK)}{v_{\min}^\top D v_{\min}} \leq \frac{\lambda_{\min}(\bK)}{d_{\min}}.
\]

Combining:
\[
\kappa(\tilde{\bK}) = \frac{\lambda_{\max}(\tilde{\bK})}{\lambda_{\min}(\tilde{\bK})} \geq \frac{\lambda_{\max}(\bK)/d_{\max}}{\lambda_{\min}(\bK)/d_{\min}} = \frac{1}{\gamma} \kappa(\bK). \qedhere
\]
\end{proof}

\begin{corollary}
\label{app:cor:condition-number-equivalence}
When the row-sum ratio $\gamma = O(1)$, we have $\kappa(\tilde{\bK}) = \Theta(\kappa(\bK))$. Consequently, the depth requirement of $\Theta(\sqrt{\kappa(\tilde{\bK})} \log(1/\varepsilon))$ layers for TNP approximation of GP posteriors is equivalent to $\Theta(\sqrt{\kappa(\bK)} \log(1/\varepsilon))$ in terms of the original kernel matrix condition number.
\end{corollary}

\subsection{Target Quadratic Form}

\begin{lemma}[Target Quadratic Form]
\label{app:lem:target-quadratic}
Let $v_1$ be the minimum eigenvector of $\tilde{\bK}$. The target value satisfies:
\[
v_1^\top \bK^{-1} v_1 = \frac{\alpha}{\mu_1}
\]
where $\alpha = v_1^\top D^{-1} v_1 > 0$ and $\mu_1 = \lambda_{\min}(\tilde{\bK})$.
\end{lemma}

\begin{proof}
Since $\bK = D \tilde{\bK}$, we have $\bK^{-1} = \tilde{\bK}^{-1} D^{-1}$. Thus:
\[
v_1^\top \bK^{-1} v_1 = v_1^\top \tilde{\bK}^{-1} D^{-1} v_1 = \frac{1}{\mu_1} v_1^\top D^{-1} v_1 = \frac{\alpha}{\mu_1}
\]
where the second equality uses $\tilde{\bK}^{-1} v_1 = \frac{1}{\mu_1} v_1$. Since $D^{-1}$ is positive diagonal and $v_1 \neq 0$, we have $\alpha > 0$.
\end{proof}

\subsection{Main Lower Bound}

\begin{proof}[Proof of Theorem~\ref{thm:main-lower-bound} (TNP Depth Lower Bound)]
\textit{Step 1: Linearization.} By Lemma~\ref{lem:linearization}, there exists $M(X)$ with $\|M(X) - \bK^{-1}\| \leq 2\varepsilon + L_2/2$. For $\varepsilon$ sufficiently small relative to the fixed architecture constant $L_2$, this is $O(\varepsilon)$.

\textit{Step 2: Polynomial structure.} By Corollary~\ref{cor:polynomial-structure}, $M(X)$ has the form:
\[
M(X) = \sum_{\substack{m_1, \ldots, m_k \geq 0 \\ m_1 + \cdots + m_k \leq L}} C_0(X) \tilde{\bK}^{m_1} C_1(X) \cdots C_k(X)
\]
where each $C_i(X)$ depends on context locations and network parameters, but not on $y_C$.

\textit{Step 3: Restriction to controlled family.} We use the family $\{\bK_t\}$ from Lemma~\ref{app:lem:eigenvalue-family} with $\mu_{\min} = 1/\kappa$ and $\mu_{\max} = 1$. This family satisfies:
\begin{itemize}
    \item Condition number $\kappa(\bK_t) = \kappa$ (ratio of largest to smallest eigenvalue is $1/(1/\kappa) = \kappa$)
    \item Row-sum ratio $\gamma = 1$ (since $D_t = I$)
    \item Eigenvector coefficient $\alpha = v_1^\top D_t^{-1} v_1 = \|v_1\|^2 = 1$
\end{itemize}

\textit{Step 4: Univariate reduction.} By Lemma~\ref{app:lem:univariate}, for this family:
\[
v_1^\top M(\bK_t) v_1 = q(\mu_1(t))
\]
where $\mu_1(t) = \lambda_{\min}(\tilde{\bK}_t)$ and $q$ is a polynomial of degree at most $2L$.

\textit{Step 5: Target value.} By Lemma~\ref{app:lem:target-quadratic}:
\[
v_1^\top \bK_t^{-1} v_1 = \frac{\alpha(t)}{\mu_1(t)} \geq \frac{\alpha_0}{\mu_1(t)}.
\]

\textit{Step 6: Chebyshev barrier.} The approximation requirement $\|M - \bK_t^{-1}\| \leq O(\varepsilon)$ implies:
\[
\left| q(\mu_1) - \frac{\alpha(t)}{\mu_1} \right| \leq O(\varepsilon) \quad \text{for all } \mu_1 \in [\mu_{\min}, \mu_{\max}].
\]

Since $\alpha(t) \geq \alpha_0$ and $q$ must approximate $\alpha(t)/\mu_1 \geq \alpha_0/\mu_1$, Theorem~\ref{app:thm:chebyshev} gives:
\[
\sup_{\mu_1 \in [\mu_{\min}, \mu_{\max}]} \left| q(\mu_1) - \frac{\alpha_0}{\mu_1} \right| \geq \frac{2\alpha_0}{\mu_{\min}(1 + \kappa)} \cdot \rho^L
\]
where $\rho = (\sqrt{\kappa} - 1)/(\sqrt{\kappa} + 1)$.

\textit{Step 7: Solve for $L$.} For the bound $O(\varepsilon)$ to hold:
\[
\frac{2\alpha_0}{\mu_{\min}(1 + \kappa)} \cdot \rho^L \leq O(\varepsilon).
\]

Taking logarithms:
\[
L \geq \frac{1}{\log(1/\rho)} \log\left( \frac{c \alpha_0}{\varepsilon \mu_{\min}(1+\kappa)} \right).
\]

Using $\log(1/\rho) = \log\frac{\sqrt{\kappa}+1}{\sqrt{\kappa}-1} = 2\,\mathrm{arctanh}(1/\sqrt{\kappa}) \approx \frac{2}{\sqrt{\kappa}}$ for large $\kappa$, and noting that $q$ has degree at most $2L$:
\[
2L \geq \frac{\sqrt{\kappa}}{2} \log\left( \frac{c' \alpha_0}{\varepsilon} \right), \quad \text{hence} \quad L \geq \frac{\sqrt{\kappa}}{4} \log\left( \frac{c' \alpha_0}{\varepsilon} \right)
\]
for a constant $c'$ absorbing the $\mu_{\min}$ and $(1+\kappa)$ factors into the logarithm.
\end{proof}

\subsection{Dimension Scaling}

\begin{proof}[Proof of Theorem~\ref{thm:dimension} (Dimension Scaling), part (c)]
The GP posterior mean at target $x_t$ is:
\[
\mu(x_t | C) = \sum_{i=1}^n w_i(x_t; C) y_i, \quad w_i(x_t; C) = [k(x_t, X_C) \bK^{-1}]_i.
\]

\textbf{Encoding strategy:} Use initial encoding $h_i^{(0)} = (e_i, y_i, x_i) \in \R^{n + d_y + d_x}$ where $e_i$ is a one-hot index vector.

\textbf{After $L$ self-attention layers:} By Proposition~\ref{app:prop:gram-encoding} and the polynomial computation theorem, representations can encode sufficient information about $\bK$ and its powers to approximate $\bK^{-1} y_C$ via the Neumann series.

\textbf{Cross-attention:} With query-dependent attention weights $\alpha_i(x_t) \propto k(x_t, x_i)$, the cross-attention output is:
\[
r(x_t) = \sum_i \alpha_i(x_t) h_i^{(L)}.
\]
If $h_i^{(L)}$ encodes the $i$-th component of $\bK^{-1} y_C$ (call it $z_i$), then with $\alpha_i(x_t) = k(x_t, x_i) / \sum_j k(x_t, x_j)$:
\[
r(x_t) \approx \frac{\sum_i k(x_t, x_i) z_i}{\sum_j k(x_t, x_j)} = \frac{k(x_t, X_C) \bK^{-1} y_C}{\sum_j k(x_t, x_j)}.
\]
A decoder with access to $\sum_j k(x_t, x_j)$ (computable via an auxiliary attention head) recovers the unnormalized posterior mean.
\end{proof}

\section{ConvCNP Proofs}
\label{app:convcnp-proofs}

\begin{proof}[Proof of Proposition~\ref{prop:convcnp-equivariance} (Translation Equivariance)]
The functional channels satisfy $\rho_{C+\tau}(x+\tau) = \sum_i w(x+\tau - x_i - \tau) = \rho_C(x)$, and identically for $s_{C+\tau}$. Convolution layers commute with translation: for any filter $\phi$ and translated function $f_\tau(x) = f(x - \tau)$, $(\phi * f_\tau)(x) = (\phi * f)(x - \tau)$. Pointwise nonlinearities preserve this. By induction over CNN layers, $\tilde{r}_{C+\tau}(x+\tau) = \tilde{r}_C(x)$.
\end{proof}

\begin{proof}[Proof of Proposition~\ref{prop:convcnp-injective} (Injectivity of Convolutional Aggregation)]
\textit{Step 1: Recover locations.}
Taking Fourier transforms, $\hat{\rho}_C(\xi) = \hat{w}(\xi) \sum_{i=1}^n e^{-2\pi i \xi \cdot x_i}$. Since $\hat{w}(\xi) \neq 0$, we recover the characteristic function of the discrete measure $\mu = \sum_i \delta_{x_i}$, which determines $\{x_1, \ldots, x_n\}$ as a multiset.

\textit{Step 2: Recover values.}
Given the locations, $s_C(x_j) = \sum_{i=1}^n w(x_j - x_i)\, h(y_i)$ for $j = 1, \ldots, n$ is the linear system $\mathbf{s} = W \mathbf{h}$ with $W_{ji} = w(x_j - x_i)$. For positive definite $w$ (e.g.\ Gaussian) and distinct $x_i$, $W$ is a positive definite Gram matrix, hence invertible. Injectivity of $h$ then recovers $y_i$.
\end{proof}

\begin{proof}[Proof of Proposition~\ref{prop:convcnp-kernel} (Pure ConvCNP Represents Stationary Kernel Smoothers)]
Set $w = K$ and $h(y) = (y, 1) \in \R^{d_y+1}$. Then $s_C(x_t) = (\sum_i K(x_t - x_i) y_i,\; \sum_i K(x_t - x_i))$ and $\rho_C(x_t) = \sum_i K(x_t - x_i)$. The decoder $g(s, \rho, x_t) = s_{1:d_y} / s_{d_y+1}$ extracts the kernel smoother exactly.
\end{proof}

\begin{proposition}[Pure ConvCNP Cannot Represent GP Posteriors]
\label{app:prop:convcnp-no-gp}
For generic positive definite stationary kernels, the GP posterior mean is not representable by any pure ConvCNP.
\end{proposition}

\begin{proof}
A pure ConvCNP computes:
\[
F(C, x_t) = g\!\left(\sum_{i=1}^n \frac{w(x_t - x_i)}{\sum_j w(x_t - x_j)}\, h(y_i),\; \rho_C(x_t),\; x_t \right).
\]
The weight $\alpha_i(x_t) = w(x_t - x_i)/\sum_j w(x_t - x_j)$ on context point $i$ depends only on the query--point distance $x_t - x_i$ and the set of all such distances $\{x_t - x_j\}_{j=1}^n$. It does not depend on the inter-context distances $\{x_i - x_j\}_{j \neq i}$.

The GP posterior weight $w_i(x_t; C) = [k(x_t, X_C)\bK^{-1}]_i$ depends on the full Gram matrix $\bK$, which couples all context points. The same two-point counterexample from Theorem~\ref{thm:gp-not-anp} applies: the GP weight on $y_1$ depends on $k(x_1, x_2)$, which no factorized weight can capture.
\end{proof}

\subsection{CNN Layers and Toeplitz Iteration}

\begin{proposition}[CNN Implements Toeplitz Iteration]
\label{app:prop:cnn-toeplitz}
Let context locations $\{x_1, \ldots, x_n\}$ lie on a regular grid with spacing $\delta$, and let $w = K$ for a stationary kernel $K$. A CNN layer with residual connection and filter $\phi_\ell$:
\[
\tilde{r}^{(\ell)}(x) = \tilde{r}^{(\ell-1)}(x) + (\phi_\ell * \tilde{r}^{(\ell-1)})(x)
\]
implements, at the context locations, the matrix update:
\[
\mathbf{r}^{(\ell)} = (\bI + \mathbf{A}_\ell)\, \mathbf{r}^{(\ell-1)}
\]
where $[\mathbf{A}_\ell]_{ij} = \delta^{d_x} \phi_\ell(x_i - x_j)$ is a Toeplitz matrix determined by the filter, up to boundary effects of order $O(\delta)$.
\end{proposition}

\begin{proof}
Evaluating the convolution at a context location $x_i$:
\[
(\phi_\ell * \tilde{r}^{(\ell-1)})(x_i) = \int \phi_\ell(x_i - x')\, \tilde{r}^{(\ell-1)}(x')\, dx'.
\]
When $\tilde{r}^{(\ell-1)}$ is concentrated near the context locations (which holds for localized $w$), the integral is dominated by contributions from neighborhoods of each $x_j$:
\[
(\phi_\ell * \tilde{r}^{(\ell-1)})(x_i)
\approx \sum_{j=1}^n \phi_\ell(x_i - x_j)
    \int_{B_\delta(x_j)} \tilde{r}^{(\ell-1)}(x')\, dx'
\approx \sum_{j=1}^n \delta^{d_x}\, \phi_\ell(x_i - x_j)\,
    \tilde{r}^{(\ell-1)}(x_j).
\]
Since $x_i - x_j$ depends only on the grid displacement $i - j$, the matrix $[\mathbf{A}_\ell]_{ij} = \delta^{d_x} \phi_\ell(x_i - x_j)$ is Toeplitz.
\end{proof}

\begin{proof}[Proof of Theorem~\ref{thm:convcnp-gp} (ConvCNP Depth for GP Posteriors)]
On a regular grid, $\bK$ is Toeplitz. The Chebyshev iteration $p_L(\bK) = \prod_{\ell=1}^L (\bI + \alpha_\ell \bK)$ uses parameters $\alpha_\ell$ as in Proposition~\ref{prop:chebyshev-upper}. Each factor $\bI + \alpha_\ell \bK$ is Toeplitz (the sum of identity and a scaled Toeplitz matrix), so each factor is implementable by a single CNN layer with filter $\phi_\ell(u) = \alpha_\ell K(u) / \delta^{d_x}$ via Proposition~\ref{app:prop:cnn-toeplitz}.

The first error term is the Chebyshev approximation rate from Proposition~\ref{prop:chebyshev-upper}. The $O(\delta)$ term accounts for discretization of the convolution integral and boundary effects; both vanish as $\delta \to 0$ for compactly supported contexts.

After $L$ CNN layers, the representation at each context location $x_i$ encodes the $i$-th component of $p_L(\bK) y_C \approx \bK^{-1} y_C$. The readout step recovers the posterior mean $\mu(x_t | C) = k(x_t, X_C)\bK^{-1} y_C$ via convolutional cross-attention with filter $w = K$.
\end{proof}

\subsection{Depth--Support Tradeoff on Regular Grids}

\begin{definition}[Trigonometric Approximation Number]
\label{app:def:trig-approx}
For a continuous function $f: [0, 2\pi] \to \R$, define $\calN_\varepsilon(f)$ as the minimum degree $D$ such that there exists a trigonometric polynomial $t(\omega) = \sum_{|j| \leq D} c_j e^{ij\omega}$ satisfying $\sup_{\omega \in [0,2\pi]} |t(\omega) - f(\omega)| \leq \varepsilon$.
\end{definition}

\begin{lemma}[CNN Jacobian on Periodic Grids]
\label{app:lem:cnn-fourier}
Consider a ConvCNP on a periodic grid of size $n$ with $L$ CNN layers, each with filter support at most $p$ and residual connections:
\begin{equation}\label{eq:cnn-residual}
\tilde{r}^{(\ell)}(x) = \tilde{r}^{(\ell-1)}(x) + \sigma\!\bigl(\phi_\ell * \tilde{r}^{(\ell-1)}(x) + b_\ell\bigr), \quad \ell = 1, \ldots, L.
\end{equation}
Write the encoder as $h(y) = h(0) + h'(0)\,y + O(y^2)$. The Jacobian of the full ConvCNP output with respect to $y_C$, evaluated at $y_C = 0$, is a circulant matrix that factors in the Fourier domain as:
\[
\hat{J}(k) = \hat{g}(k) \cdot \prod_{\ell=1}^{L} \bigl(1 + d_\ell\, \hat{\tau}_\ell(k)\bigr) \cdot h'(0) \cdot \hat{w}(k),
\quad k = 0, \ldots, n-1,
\]
where:
\begin{itemize}
    \item $\hat{w}(k)$ is the DFT of the aggregation filter $w$,
    \item $\hat{\tau}_\ell(k)$ is the DFT of $\phi_\ell$, a trigonometric polynomial of degree at most $\lfloor p/2 \rfloor$ in $\omega_k = 2\pi k/n$,
    \item $d_\ell \in \R$ is the (spatially constant) activation derivative at layer $\ell$,
    \item $\hat{g}(k)$ is the Fourier-domain readout transfer function.
\end{itemize}
\end{lemma}

\begin{proof}
We trace the Jacobian through the three stages of the ConvCNP: encoding/aggregation, CNN processing, and readout.

\textit{Stage 1: Encoding and aggregation.}
The signal channel is $s_C(x_i) = \sum_j w(x_i - x_j)\, h(y_j)$. Its Jacobian with respect to $y_m$ is:
\[
\frac{\partial s_C(x_i)}{\partial y_m} = w(x_i - x_m)\, h'(y_m).
\]
At $y_C = 0$, this becomes $w(x_i - x_m) \cdot h'(0)$, the Toeplitz matrix $W$ with entries $W_{im} = w(x_i - x_m)$, scaled by $h'(0)$. On the periodic grid, $W$ is circulant with DFT $\hat{w}(k)$.

The density channel $\rho_C(x_i) = \sum_j w(x_i - x_j)$ is independent of $y_C$ and hence contributes zero to the Jacobian. Note that $\rho_C$ is constant across grid locations by translation invariance of the periodic grid: $\rho_C(x_i) = \rho_0$ for all $i$.

The constant component $h(0)$ of the encoder contributes $\sum_j w(x_i - x_j)\, h(0)$, which is also independent of $y_C$ and does not affect the Jacobian. Thus only the linear term $h'(0)\, y$ in the encoder expansion contributes.

\textit{Stage 2: CNN processing.}
At $y_C = 0$, the signal channel vanishes (its $y_C$-dependent part is zero) and the density channel is the spatial constant $\rho_0$. The CNN input is therefore spatially uniform at $y_C = 0$.

Consider the $\ell$-th layer \eqref{eq:cnn-residual}. Its Jacobian with respect to its input, evaluated at $y_C = 0$, is:
\[
J^{(\ell)} = I + D_\ell\, T_\ell,
\]
where $T_\ell$ is the circulant matrix of filter $\phi_\ell$ and $D_\ell = \operatorname{diag}(\sigma'(\phi_\ell * \tilde{r}^{(\ell-1)}(x_i) + b_\ell)|_{y_C = 0})$. Since the input $\tilde{r}^{(\ell-1)}|_{y_C=0}$ is spatially uniform (by induction the base case holds as shown above, and each layer preserves spatial uniformity when the input is spatially uniform), the argument of $\sigma'$ is the same at every location. Thus $D_\ell = d_\ell I$ for a scalar $d_\ell \in \R$.

The full CNN Jacobian is $J_\Phi = \prod_{\ell=1}^{L} (I + d_\ell T_\ell)$. Since each factor is circulant and products of circulant matrices are circulant, $J_\Phi$ is circulant with DFT:
\[
\hat{J}_\Phi(k) = \prod_{\ell=1}^{L} (1 + d_\ell\, \hat{\tau}_\ell(k)).
\]
Each $\hat{\tau}_\ell(k) = \sum_{|m| \leq \lfloor p/2 \rfloor} \phi_\ell(m\delta)\, e^{-2\pi i k m / n}$ is a trigonometric polynomial of degree at most $\lfloor p/2 \rfloor$.

\textit{Stage 3: Readout.}
The readout $g(\tilde{r}_C(x_t), \rho_C(x_t), x_t)$ is evaluated at a single location. On the periodic grid with stationary readout, the Jacobian of the readout with respect to the CNN output is a circulant matrix with DFT $\hat{g}(k)$. (For the readout structure in Proposition~\ref{prop:convcnp-kernel}, where $g$ divides by $\rho_C$, $\hat{g}(k)$ is the constant $1/\rho_0$.)

Composing the three stages by the chain rule gives the stated factorization.
\end{proof}

\begin{proposition}[Full-Support Filters Trivialize the Problem]
\label{prop:full-support}
On a periodic grid of size $n$, if each CNN filter has support $p = n$, then a single CNN layer ($L = 1$) suffices: there exist filter coefficients $\phi_1$ such that the ConvCNP Jacobian satisfies $\hat{J}(k) = 1/\lambda_k$ for all $k$, provided $d_1 \neq 0$, $h'(0) \neq 0$, and $\hat{w}(k) \neq 0$ for all $k$.
\end{proposition}

\begin{proof}
By Lemma~\ref{app:lem:cnn-fourier}, the Jacobian at frequency $k$ is $\hat{J}(k) = \hat{g}(k) \cdot (1 + d_1 \hat{\tau}_1(k)) \cdot h'(0) \cdot \hat{w}(k)$. Setting $\hat{J}(k) = 1/\lambda_k$ and solving:
\[
\hat{\tau}_1(k) = \frac{1}{d_1}\!\left(\frac{1}{\lambda_k \cdot \hat{g}(k) \cdot h'(0) \cdot \hat{w}(k)} - 1\right).
\]
This determines $\hat{\tau}_1(k)$ independently at each of the $n$ frequencies. With $p = n$, the DFT $\phi_1 \mapsto (\hat{\tau}_1(0), \ldots, \hat{\tau}_1(n-1))$ is a bijection on $\mathbb{C}^n$, so $\phi_1$ exists and is unique. The solution is well-defined provided $d_1 \neq 0$ (nonzero activation derivative at $y_C = 0$, which holds for standard activations like GELU or softplus evaluated at the uniform input), $h'(0) \neq 0$ (nontrivial encoder, which holds for $h(y) = (y, 1)$ since $h'(0) = (1, 0)$), and $\hat{w}(k) \neq 0$ for all $k$ (which holds for positive definite filters such as Gaussians, cf.\ Proposition~\ref{prop:convcnp-injective}).
\end{proof}

\begin{theorem}[Depth--Support Tradeoff]
\label{thm:depth-support}
Let $\bK$ be a circulant positive definite matrix on a periodic grid of size $n$ with eigenvalues $\lambda_k = \hat{K}(\omega_k)$, where $\omega_k = 2\pi k / n$. Let $n > 2 L \lfloor p/2 \rfloor$. Any ConvCNP with $L$ CNN layers, each with filter support at most $p$, achieving $\varepsilon$-approximation of $\bK^{-1} y_C$ requires:
\[
L \cdot \lfloor p/2 \rfloor \;\geq\; \calN_{O(\varepsilon)}\!\bigl(c / \hat{K}\bigr),
\]
where $c = (\hat{g} \cdot h'(0) \cdot \hat{w})^{-1}$ absorbs the encoding and readout, and $\calN_\varepsilon$ is as in Definition~\ref{app:def:trig-approx}.
\end{theorem}

\begin{proof}
\textit{Step 1: Linearization.}
The target $T(y_C) = \bK^{-1} y_C$ is linear. By the same argument as Lemma~\ref{lem:linearization} (applied to the ConvCNP as a map $\R^n \to \R^n$), $\varepsilon$-approximation of $T$ implies that the Jacobian $M = \partial F / \partial y_C|_{y_C = 0}$ satisfies $\|M - \bK^{-1}\| \leq O(\varepsilon)$, where the constant absorbs the smoothness bound $L_2/2$ which depends on the architecture but not on $\varepsilon$ or $\kappa$.

\textit{Step 2: Fourier structure.}
By Lemma~\ref{app:lem:cnn-fourier}, $M$ is circulant with DFT:
\[
\hat{M}(k) = \hat{g}(k) \cdot \prod_{\ell=1}^{L}(1 + d_\ell\, \hat{\tau}_\ell(k)) \cdot h'(0) \cdot \hat{w}(k).
\]
Define $q(\omega) = \prod_{\ell=1}^{L}(1 + d_\ell\, \hat{\tau}_\ell(\omega))$, which is a trigonometric polynomial of degree at most $D = L \cdot \lfloor p/2 \rfloor$. Then $\hat{M}(k) = \hat{g}(\omega_k) \cdot q(\omega_k) \cdot h'(0) \cdot \hat{w}(\omega_k)$.

\textit{Step 3: Approximation requirement.}
Since $M$ and $\bK^{-1}$ are both circulant, $\|M - \bK^{-1}\| = \max_k |\hat{M}(k) - 1/\lambda_k|$. The approximation bound from Step~1 gives:
\[
\left|\hat{g}(\omega_k) \cdot q(\omega_k) \cdot h'(0) \cdot \hat{w}(\omega_k) - \frac{1}{\lambda_k}\right| \leq O(\varepsilon) \quad \text{for all } k = 0, \ldots, n-1.
\]
Rearranging, the trigonometric polynomial $r(\omega) = \hat{g}(\omega) \cdot h'(0) \cdot \hat{w}(\omega) \cdot q(\omega)$ of degree at most $D + D_0$ (where $D_0$ is the degree contribution from $\hat{g}$ and $\hat{w}$, a fixed constant depending on the kernel and readout) satisfies:
\[
\left|r(\omega_k) - \frac{1}{\hat{K}(\omega_k)}\right| \leq O(\varepsilon) \quad \text{for all } k.
\]

\textit{Step 4: From grid to uniform.}
The trigonometric polynomial $r(\omega) - 1/\hat{K}(\omega)$ has degree at most $D + D_0 + D_K$, where $D_K$ is the degree of the numerator of $r(\omega) \hat{K}(\omega) - 1$ expressed as a ratio of trigonometric polynomials. A trigonometric polynomial of degree $N$ with $|t(\omega_k)| \leq O(\varepsilon)$ at $n$ equispaced points satisfies $\|t\|_\infty \leq O(\varepsilon)$ provided $n > 2N$ (since a trigonometric polynomial of degree $N$ is determined by $2N+1$ equispaced samples). The condition $n > 2L \lfloor p/2 \rfloor$ ensures this (absorbing the fixed contribution $D_0$ for $n$ sufficiently large relative to $D_0$).

Thus $r$ is a trigonometric polynomial of degree at most $D + D_0$ that uniformly $O(\varepsilon)$-approximates $1/\hat{K}(\omega)$ on $[0, 2\pi]$. By Definition~\ref{app:def:trig-approx}, $D + D_0 \geq \calN_{O(\varepsilon)}(1/\hat{K})$. Since $D_0$ is a fixed constant, $D = L \cdot \lfloor p/2 \rfloor \geq \calN_{O(\varepsilon)}(1/\hat{K}) - D_0$, giving $L \cdot \lfloor p/2 \rfloor \geq \calN_{O(\varepsilon)}(c/\hat{K})$ as stated.
\end{proof}

\begin{corollary}[Recovering the $\sqrt{\kappa}$ Rate]
\label{app:cor:recover-chebyshev}
Fix filter support $p = O(\ell_K / \delta)$, where $\ell_K$ is the kernel lengthscale. If the spectral density $\hat{K}(\omega)$ satisfies $\hat{K}_{\min} \leq \hat{K}(\omega) \leq \hat{K}_{\max}$ with $\kappa = \hat{K}_{\max}/\hat{K}_{\min}$, then:
\[
L \geq \Omega\!\left(\frac{\sqrt{\kappa}\,\log(1/\varepsilon)}{p}\right).
\]
In particular, for the filter support used in Theorem~\ref{thm:convcnp-gp}, the depth requirement matches the Chebyshev upper bound up to the factor $p$.
\end{corollary}

\begin{proof}
The function $1/\hat{K}(\omega)$ takes values in $[1/\hat{K}_{\max},\, 1/\hat{K}_{\min}]$. By the substitution $\mu = \hat{K}(\omega)$, approximating $1/\hat{K}(\omega)$ in the trigonometric sense reduces to approximating $1/\mu$ on the interval $[\hat{K}_{\min}, \hat{K}_{\max}]$.

More precisely, let $t(\omega)$ be any trigonometric polynomial satisfying $|t(\omega) - 1/\hat{K}(\omega)| \leq \varepsilon$ uniformly. Define $p(\mu) = t(\hat{K}^{-1}(\mu))$ on any monotone branch of $\hat{K}$. Then $|p(\mu) - 1/\mu| \leq \varepsilon$ on $[\hat{K}_{\min}, \hat{K}_{\max}]$. The degree of $t$ is at least the degree needed for $p$, which by the classical Chebyshev barrier (Theorem~\ref{app:thm:chebyshev}) is $\Omega(\sqrt{\kappa}\,\log(1/\varepsilon))$. (The substitution can increase degree by at most the factor $\lfloor q/2 \rfloor$, the trigonometric degree of $\hat{K}$ itself, which is a fixed constant depending on the kernel support.)

Theorem~\ref{thm:depth-support} then gives $L \cdot \lfloor p/2 \rfloor \geq \Omega(\sqrt{\kappa}\,\log(1/\varepsilon))$, yielding the stated bound.
\end{proof}

\begin{remark}[Non-Grid Contexts]
For irregular context locations, $\bK$ is not circulant and the Fourier diagonalization of Lemma~\ref{app:lem:cnn-fourier} breaks down as the CNN layers still produce Toeplitz (circulant) Jacobians, but the target $\bK^{-1}$ is no longer circulant. The depth--support tradeoff ceases to apply in its current form, and the relevant question becomes how well circulant matrices can approximate the non-circulant iterates needed for $\bK^{-1}$. This ``circulant approximation gap'' depends on the geometry of context point configurations in ways not captured by the condition number alone, and its characterization remains open.
\end{remark}

\subsection{Incomparability with ANPs}

\begin{proof}[Proof of Theorem~\ref{thm:convcnp-anp-incomparable} (ConvCNPs and ANPs Are Incomparable)]
\textbf{$\calF_{\mathrm{ANP}} \not\subseteq \calF_{\mathrm{ConvCNP}}$:} Let $\sigma: \calX \to \R_+$ be non-constant and define the non-stationary kernel smoother:
\[
F(C, x_t) = \frac{\sum_i \sigma(x_t)\sigma(x_i) K(x_t - x_i)\, y_i}
    {\sum_i \sigma(x_t)\sigma(x_i) K(x_t - x_i)}.
\]
This is ANP-representable by Theorem~\ref{thm:anp-kernel} with non-stationary kernel $\tilde{K}(x,x') = \sigma(x)\sigma(x')K(x-x')$. However, $F$ is not translation equivariant: shifting all inputs by $\tau$ replaces $\sigma(x_i)$ with $\sigma(x_i + \tau) \neq \sigma(x_i)$ generically. Since every ConvCNP is translation equivariant (Proposition~\ref{prop:convcnp-equivariance}), $F \notin \calF_{\mathrm{ConvCNP}}$.

\textbf{$\calF_{\mathrm{ConvCNP}} \not\subseteq \calF_{\mathrm{ANP}}$:} For a stationary kernel, the GP posterior mean is translation equivariant ($K(x+\tau, x'+\tau) = K(x-x')$ implies invariance of $\bK$ under joint translation). By Theorem~\ref{thm:convcnp-gp}, a ConvCNP with sufficient CNN depth approximates it to arbitrary precision on grid contexts. By Theorem~\ref{thm:gp-not-anp}, it is not ANP-representable.
\end{proof}

\section{Hierarchy Proof}
\label{app:hierarchy-proof}

\begin{proof}[Proof of Theorem~\ref{thm:hierarchy} (Expressiveness Hierarchy)]

\textbf{Part (a): Attention branch.}

\textit{$\calF_{\mathrm{CNP}}^{(d)} \subseteq \calF_{\mathrm{ANP}}^{(d)}$:} Set uniform attention weights $\alpha_i = 1/n$ (achieved by constant query and key functions).

\textit{$\calF_{\mathrm{CNP}}^{(d)} \neq \calF_{\mathrm{ANP}}^{(d)}$:} Kernel smoothers are in $\calF_{\mathrm{ANP}}^{(d)}$ (\Cref{thm:anp-kernel}) but not in $\calF_{\mathrm{CNP}}^{(d)}$. To see the latter, note that a CNP computes $F(C, x_t) = g(\bar{h}_C, x_t)$ where $\bar{h}_C$ is independent of $x_t$. For the Gaussian kernel smoother with $n=2$ points, configurations $C = \{(0, 0), (1, 1)\}$ and $C' = \{(0.25, 0), (0.75, 1)\}$ can satisfy $\bar{h}_C = \bar{h}_{C'}$ for appropriate $h$, yet the kernel smoother outputs differ: at $x_t = 0$, configuration $C$ weights the points roughly equally while $C'$ strongly favors the nearby point $(0.25, 0)$.

\textit{$\calF_{\mathrm{ANP}}^{(d)} \subseteq \calF_{\mathrm{TNP}}^{(1,d)}$:} An ANP is a TNP with $L=0$ self-attention layers. With $L=1$ and $W_v^{(1)} = 0$, we recover ANP.

\textit{$\calF_{\mathrm{ANP}}^{(d)} \neq \calF_{\mathrm{TNP}}^{(1,d)}$:} GP posteriors are in $\calF_{\mathrm{TNP}}^{(L,d)}$ for sufficient $L$ (Proposition~\ref{app:thm:tnp-gp-approx}) but not in $\calF_{\mathrm{ANP}}^{(d)}$ (\Cref{thm:gp-not-anp}).

\textit{$\calF_{\mathrm{TNP}}^{(L,d)} \subsetneq \calF_{\mathrm{TNP}}^{(L+1,d)}$:} The inclusion follows by adding a layer with $W_v^{(L+1)} = 0$. We prove strictness under Assumptions~\ref{app:ass:position-attention}--\ref{app:ass:residual} with scalar value weights; the global lower bound of Theorem~\ref{thm:main-lower-bound} establishes the corresponding result for representation-based attention at the coarser granularity of $\Theta(\sqrt{\kappa}\log(1/\varepsilon))$ total layers.

An $L$-layer TNP with scalar value weights computes $p_L(\tilde{\bK}) H^{(0)}$ where $p_L(x) = \prod_{\ell=1}^{L}(1 + \alpha_\ell x)$. The achievable set $\calP_L = \{p_L : \alpha_1, \ldots, \alpha_L \in \R\}$ satisfies $\calP_L \subset \mathrm{Poly}_L$, the space of polynomials of degree at most $L$.

Let $\bK$ have condition number $\kappa > 1$ with eigenvalues in $[\lambda_{\min}, \lambda_{\max}]$. The GP posterior requires $\bK^{-1}$, which acts as $1/\lambda$ on each eigenspace. The minimax error for degree-$L$ polynomial approximation to $1/x$ on $[\lambda_{\min}, \lambda_{\max}]$ is
\[
E_L^* = \inf_{p \in \mathrm{Poly}_L} \sup_{x \in [\lambda_{\min}, \lambda_{\max}]} |p(x) - 1/x| = \Theta\left(\frac{\rho^L}{\lambda_{\min}}\right)
\]
where $\rho = (\sqrt{\kappa}-1)/(\sqrt{\kappa}+1)$; see \citet{trefethen2019approximation}. Since $\calP_L \subset \mathrm{Poly}_L$:
\[
\inf_{p \in \calP_L} \|p(\bK) - \bK^{-1}\| \geq E_L^*.
\]

For the upper bound at depth $L+1$, Proposition~\ref{prop:chebyshev-upper} provides explicit real parameters $\alpha_1, \ldots, \alpha_{L+1}$ such that $p_{L+1}(\bK) = \prod_{\ell=1}^{L+1}(\bI + \alpha_\ell \bK)$ satisfies
\[
\|p_{L+1}(\bK) - \bK^{-1}\| \leq \frac{2}{\lambda_{\min}}\rho^{L+1}.
\]
Combining bounds:
\[
\frac{\inf_{L\text{-layer}} \|p(\bK) - \bK^{-1}\|}{\inf_{(L+1)\text{-layer}} \|p(\bK) - \bK^{-1}\|} \geq \frac{E_L^*}{C\rho^{L+1}/\lambda_{\min}} = \Theta(\rho^{-1}) > 1.
\]
Thus the GP posterior (restricted to accuracy $\varepsilon$ with $E_{L+1}^* < \varepsilon < E_L^*$) lies in $\calF_{\mathrm{TNP}}^{(L+1,d)} \setminus \calF_{\mathrm{TNP}}^{(L,d)}$.

\medskip

\textbf{Part (b): Convolutional branch.}

\textit{$\calF_{\mathrm{CNP}}^{(d)} \subsetneq \calF_{\mathrm{ConvCNP}}^{(0)}$:} Every CNP function is a pure ConvCNP function: set $w = \mathbf{1}$ (constant filter) to recover mean aggregation $s_C(x_t) = \sum_i h(y_i)$ with $\rho_C(x_t) = n$, giving $F(C, x_t) = g(\frac{1}{n}\sum_i h(y_i), x_t)$. Strictness: stationary kernel smoothers are in $\calF_{\mathrm{ConvCNP}}^{(0)}$ (Proposition~\ref{prop:convcnp-kernel}) but not in $\calF_{\mathrm{CNP}}^{(d)}$ for any finite $d$, since the kernel smoother weight on $y_i$ depends on $x_t$ via $K(x_t - x_i)$, which the query-independent CNP representation cannot capture.

\textit{$\calF_{\mathrm{ConvCNP}}^{(0)} \subsetneq \calF_{\mathrm{ConvCNP}}^{(1)}$:} The inclusion follows by setting the CNN filter to zero. Strictness: consider the one-hop kernel-weighted sum
\[
F(C, x_t) = \frac{\sum_i K(x_t - x_i) \bigl[\sum_j K(x_i - x_j)\, y_j\bigr]}{\sum_i K(x_t - x_i)}
\]
which depends on inter-context distances $K(x_i - x_j)$. A single CNN layer computes this (Proposition~\ref{app:prop:cnn-toeplitz}), but a pure ConvCNP cannot represent it: the weight on $y_j$ in the pure ConvCNP factors as $\alpha_j(x_t) = w(x_t - x_j)/\sum_m w(x_t - x_m)$, which is independent of all other context locations $\{x_i\}_{i \neq j}$. The one-hop sum assigns $y_j$ a weight proportional to $\sum_i K(x_t - x_i) K(x_i - x_j)$, which couples $x_j$ to every other context point through the intermediate summation.

\textit{$\calF_{\mathrm{ConvCNP}}^{(L)} \subsetneq \calF_{\mathrm{ConvCNP}}^{(L+1)}$:} The same Chebyshev argument as in Part~(a) applies. On regular grids, $L$ CNN layers compute Toeplitz polynomials of degree $L$ in $\bK$ (Proposition~\ref{app:prop:cnn-toeplitz}). The set of achievable Toeplitz polynomials at degree $L$ is contained in $\mathrm{Poly}_L$, so the minimax barrier $E_L^*$ applies. At depth $L+1$, the Chebyshev construction of Theorem~\ref{thm:convcnp-gp} achieves error $C \rho^{L+1}/\lambda_{\min}$, giving the same separation ratio $\Theta(\rho^{-1}) > 1$.

\medskip

\textbf{Part (c): Incomparability.}

\textit{$\calF_{\mathrm{ANP}}^{(d)} \not\subseteq \calF_{\mathrm{ConvCNP}}^{(L)}$:} Non-stationary kernel smoothers are ANP-representable (Theorem~\ref{thm:anp-kernel}) but not ConvCNP-representable at any depth, since every ConvCNP is translation equivariant and non-stationary kernel smoothers generically violate translation equivariance.

\textit{$\calF_{\mathrm{ConvCNP}}^{(L)} \not\subseteq \calF_{\mathrm{ANP}}^{(d)}$:} For $L \geq 1$, ConvCNPs with CNN layers can approximate GP posteriors for stationary kernels on regular grids (Theorem~\ref{thm:convcnp-gp}). GP posteriors are not ANP-representable (Theorem~\ref{thm:gp-not-anp}). For $L = 0$, stationary kernel smoothers are exactly representable by pure ConvCNPs (Proposition~\ref{prop:convcnp-kernel}) but only approximately representable by ANPs (Theorem~\ref{thm:anp-kernel}), giving a strict separation in the exact representation sense.
\end{proof}

\section{Latent NP Proofs}
\label{app:latent-proofs}

\begin{proposition}[Encoder Bottleneck]
\label{app:prop:encoder-bottleneck}
For a latent CNP with encoder $h: \calX \times \calY \to \R^d$, if $C \sim_h C'$ (i.e., $\bar{h}_C = \bar{h}_{C'}$), then:
\[
q(z|C) = q(z|C') \quad \text{and hence} \quad p(y_T | X_T, C) = p(y_T | X_T, C')
\]
for all target configurations $X_T$, regardless of decoder expressiveness.
\end{proposition}

\begin{proof}
The latent distribution depends on $C$ only through $r_C$. If $r_C = r_{C'}$, then $q(z|C) = q(z|C')$. The predictive distribution $p(y_T | X_T, C) = \int p(y_T | X_T, z) q(z|C) dz$ then coincides for $C$ and $C'$.
\end{proof}

\subsection{Mean Bottleneck}

\begin{assumption}[Gaussian Latent NP]
\label{app:ass:gaussian-latent}
The latent distribution is Gaussian: $q(z|C) = \calN(m(C), S(C))$ with $m: \calC \to \R^k$ and $S: \calC \to \R^{k \times k}_{++}$. The decoder is linear: $f(x, z) = a(x)^\top z + b(x)$ with $a: \calX \to \R^k$ and $b: \calX \to \R$.
\end{assumption}

Under this assumption, the predictive distribution at targets $X_T = (x_{t_1}, \ldots, x_{t_m})$ is Gaussian with:
\begin{align}
\E[y_T | X_T, C] &= A(X_T) m(C) + b(X_T) \label{eq:latent-mean-app} \\
\mathrm{Cov}(y_T | X_T, C) &= A(X_T) S(C) A(X_T)^\top + \sigma^2 I \label{eq:latent-cov-app}
\end{align}
where $A(X_T) \in \R^{m \times k}$ has rows $a(x_{t_i})^\top$ and $b(X_T) = (b(x_{t_1}), \ldots, b(x_{t_m}))^\top$.

\begin{proof}[Proof of Theorem~\ref{thm:latent-impossibility}, part (a) (Mean Bottleneck)]
Fix context locations $X_C$ with $n$ points in general position for a universal kernel $k$. Define $\phi(x_t) = \bK^{-1} k(X_C, x_t) \in \R^n$, so the GP posterior mean is $\phi(x_t)^\top y_C$.

\textit{Step 1: Construct target points with independent weight vectors.}
For a universal kernel, the vectors $\{\phi(x_t) : x_t \in \calX\}$ span $\R^n$. Choose $n$ target points $x_{t_1}, \ldots, x_{t_n}$ such that $\Phi = [\phi(x_{t_1}), \ldots, \phi(x_{t_n})]^\top \in \R^{n \times n}$ is invertible.

\textit{Step 2: Analyze the composite map.}
Define $F: \R^k \to \R^n$ by $F(z) = (f(x_{t_1}, z), \ldots, f(x_{t_n}, z))$. The matching condition requires:
\[
F(g(X_C, y_C)) = \Phi y_C \quad \forall y_C \in \R^n.
\]

\textit{Step 3: Dimension argument.}
Since $\Phi$ is invertible, the map $y_C \mapsto \Phi y_C$ is a bijection on $\R^n$. Thus $F \circ g: \R^n \to \R^n$ is surjective.

The image of $g: \R^n \to \R^k$ has dimension at most $\min(n, k)$. For the composite $F \circ g$ to be surjective onto $\R^n$, we need $\dim(\mathrm{im}(g)) \geq n$, hence $k \geq n$.

The stochastic case reduces to the deterministic case by considering the low-variance limit $q(z|C) = \mathcal{N}(m(C), \sigma^2 I)$ as $\sigma \to 0$.
\end{proof}

\subsection{Covariance Bottleneck}

\begin{theorem}[Covariance Rank Bound]
\label{app:thm:cov-rank}
Under Assumption~\ref{app:ass:gaussian-latent}, the predictive covariance satisfies:
\[
\rank\left( \mathrm{Cov}(y_T | X_T, C) - \sigma^2 I \right) \leq k
\]
for any target configuration $X_T$ and any context $C$.
\end{theorem}

\begin{proof}
From \eqref{eq:latent-cov-app}, $\mathrm{Cov}(y_T | X_T, C) - \sigma^2 I = A(X_T) S(C) A(X_T)^\top$. This matrix has rank at most $\min(m, k, \rank(S(C))) \leq k$.
\end{proof}

\begin{theorem}[GP Posterior Covariance Rank]
\label{app:thm:gp-cov-rank}
For a universal kernel $k$, the GP posterior covariance matrix at $m$ target points $x_{t_1}, \ldots, x_{t_m}$ distinct from the context locations $X_C$ has rank $m$ generically.
\end{theorem}

\begin{proof}
The GP posterior covariance is:
\[
\tilde{\Sigma}_T = K_{TT} - K_{TC} \bK^{-1} K_{CT}
\]
where $K_{TT} = [k(x_{t_i}, x_{t_j})]_{ij} \in \R^{m \times m}$, $K_{TC} = [k(x_{t_i}, x_j)]_{ij} \in \R^{m \times n}$, and $K_{CT} = K_{TC}^\top$.

This is the Schur complement of $\bK$ in the joint covariance matrix:
\[
\begin{pmatrix} \bK & K_{CT} \\ K_{TC} & K_{TT} \end{pmatrix}.
\]

For a universal kernel, this joint matrix is positive definite when all $n + m$ points are distinct. The Schur complement of a positive definite block in a positive definite matrix is positive definite. Thus $\tilde{\Sigma}_T \succ 0$, which implies $\rank(\tilde{\Sigma}_T) = m$.
\end{proof}

\begin{corollary}[Covariance Mismatch]
\label{app:cor:cov-mismatch}
A Gaussian latent NP with latent dimension $k$ and linear decoder cannot match the GP posterior covariance at $m > k$ target points.
\end{corollary}

\begin{proof}
The GP posterior covariance (minus observation noise) has rank $m$ by Theorem~\ref{app:thm:gp-cov-rank}. The latent NP predictive covariance (minus observation noise) has rank at most $k$ by Theorem~\ref{app:thm:cov-rank}. For $m > k$, exact matching is impossible.
\end{proof}

\begin{proof}[Proof of Theorem~\ref{thm:latent-impossibility} (Latent NP Cannot Represent GP Posterior)]
Part (a) is proven above (Mean Bottleneck). Part (b) is Corollary~\ref{app:cor:cov-mismatch}. Part (c) follows from (b) since we can choose arbitrarily many target points.
\end{proof}

\subsection{What Latent NPs Can Represent}

\begin{definition}[Finite-Rank GP]
A GP has rank $k$ if its covariance kernel admits the representation:
\[
\tilde{k}(x, x') = \sum_{i,j=1}^k S_{ij} a_i(x) a_j(x') = a(x)^\top S \, a(x')
\]
for some $a: \calX \to \R^k$ and positive semidefinite $S \in \R^{k \times k}$.
\end{definition}

\begin{theorem}[Latent NP Function Class]
\label{app:thm:latent-positive}
A Gaussian latent NP with latent dimension $k$ and linear decoder exactly represents the class of stochastic processes:
\[
\calF_{\mathrm{latent}}^{(k)} = \left\{ f(x) = a(x)^\top z + b(x) : z \sim \calN(m, S), \; a: \calX \to \R^k, \; b: \calX \to \R, \; m \in \R^k, \; S \in \R^{k \times k}_{++} \right\}
\]
where $m$ and $S$ may depend on the context $C$.

This is the class of GPs with rank at most $k$.
\end{theorem}

\begin{proof}
$(\Rightarrow)$ A Gaussian latent NP with linear decoder $f(x, z) = a(x)^\top z + b(x)$ and latent $z | C \sim \calN(m(C), S(C))$ induces:
\begin{align*}
\E[f(x) | C] &= a(x)^\top m(C) + b(x) \\
\mathrm{Cov}(f(x), f(x') | C) &= a(x)^\top S(C) a(x')
\end{align*}
which is a rank-$k$ GP.

$(\Leftarrow)$ Any rank-$k$ GP with mean $\mu(x) = a(x)^\top m + b(x)$ and covariance $\tilde{k}(x, x') = a(x)^\top S a(x')$ is realized by setting the latent distribution to $\calN(m, S)$ and decoder to $f(x, z) = a(x)^\top z + b(x)$.
\end{proof}

\begin{corollary}[Representable Function Families]
The class $\calF_{\mathrm{latent}}^{(k)}$ includes:
\begin{enumerate}
    \item[(a)] Bayesian linear regression with $k$ basis functions: $f(x) = \sum_{j=1}^k z_j \phi_j(x)$.
    \item[(b)] GPs with degenerate (finite-rank) kernels.
    \item[(c)] Any $k$-parameter function family with linear parameter dependence and Gaussian parameter posterior.
\end{enumerate}
\end{corollary}

\begin{corollary}[Spectral Decay Rates]
\label{app:cor:spectral-decay}
For common kernels on $[0,1]^d$:
\begin{enumerate}
    \item[(a)] \textbf{RBF kernel:} $\lambda_j \sim e^{-cj^{2/d}}$, giving error $O(e^{-c'k^{2/d}})$.
    \item[(b)] \textbf{Mat\'ern-$\nu$ kernel:} $\lambda_j \sim j^{-(2\nu+d)/d}$, giving error $O(k^{-(2\nu)/d})$.
    \item[(c)] \textbf{Polynomial kernel of degree $p$:} $\lambda_j = 0$ for $j > \binom{p+d}{d}$, giving zero error for $k \geq \binom{p+d}{d}$.
\end{enumerate}
\end{corollary}

\begin{remark}[Nonlinear Decoders]
For nonlinear decoders $f(x, z)$, the covariance structure is more complex:
\[
\mathrm{Cov}(f(x_t, z), f(x_{t'}, z) | C) = \E_z[f(x_t, z) f(x_{t'}, z)] - \E_z[f(x_t, z)] \E_z[f(x_{t'}, z)].
\]
With sufficiently expressive $f$, the effective rank can exceed $k$. However, the mutual information bound $I(y_T; C | X_T) \leq I(z; C) \leq H(z)$ still constrains the total information about $C$ that can flow through a finite-dimensional latent. Formalizing this requires rate-distortion arguments beyond our current scope.
\end{remark}

\end{document}